\documentclass[letterpaper]{article} 
\usepackage{aaai25}  
\usepackage{times}  
\usepackage{helvet}  
\usepackage{courier}  
\usepackage[hyphens]{url}  
\usepackage{graphicx} 
\usepackage{natbib}  
\usepackage{caption} 
\usepackage{algorithm}

\usepackage{newfloat}
\usepackage{listings}
\DeclareCaptionStyle{ruled}{labelfont=normalfont,labelsep=colon,strut=off} 
\floatstyle{ruled}
\newfloat{listing}{tb}{lst}{}
\floatname{listing}{Listing}
\usepackage{xcolor,soul,framed} 

\colorlet{shadecolor}{yellow}
\usepackage{multirow}
\usepackage{array}
\usepackage{amsthm}

\newtheoremstyle{theorem}
  {\topsep}
  {\topsep}
  {}
  {}
  {\itshape}
  {:}
  {.5em}
  {\thmname{#1}\thmnumber{ #2}\thmnote{ (#3)}}
\theoremstyle{theorem}

\newtheoremstyle{proposition}
  {\topsep}
  {\topsep}
  {}
  {}
  {\itshape}
  {:}
  {.5em}
  {\thmname{#1}\thmnumber{ #2}\thmnote{ (#3)}}
\theoremstyle{proposition}

\newtheorem{theorem}{Theorem}
\newtheorem{lemma}{\bf Lemma}
\newtheorem{proposition}{Proposition}

\newtheorem{myDef}{Definition}

\usepackage{mdwmath}
\usepackage{mdwtab}
\usepackage{eqparbox}
\usepackage{url}

\usepackage{amssymb}
\usepackage{amsfonts}
\usepackage{amsmath}
\usepackage{multicol}
\usepackage{bm}

\usepackage{algorithm}
\usepackage{algorithmicx}
\usepackage{algpseudocode}
\usepackage{cancel}
\usepackage{subeqnarray}
\usepackage{cases}
\usepackage{mathrsfs}
\usepackage{subfigure}
\usepackage{latexsym}
\usepackage{booktabs}
\usepackage{enumitem}

\title{\textsc{Grimm}: A Plug-and-Play Perturbation Rectifier for Graph Neural Networks Defending against Poisoning Attacks}
\author{
    Ao Liu\textsuperscript{\rm 1},
    Wenshan Li\textsuperscript{\rm 2}\thanks{Corresponding author},
    Beibei Li\textsuperscript{\rm 1},
    Wengang Ma\textsuperscript{\rm 1},
    Tao Li\textsuperscript{\rm 1},
    Pan Zhou\textsuperscript{\rm 3},
}
\affiliations{
    \textsuperscript{\rm 1}School of Cyber Science and Engineering, Sichuan University \\
    \textsuperscript{\rm 2}School of Cyber Science and Engineering, Chengdu University of Information Technology \\
    \textsuperscript{\rm 3}School of Cyber Science and Engineering, Huazhong University of Science and Technology \\

    \{aliu, libeibei, mawengang941206, litao\}@scu.edu.cn,
    helenali@cuit.edu.cn,
    panzhou@hust.edu.cn,
}

\usepackage{bibentry}

\begin{document}

\maketitle

\begin{abstract}

Recent studies have revealed the vulnerability of graph neural networks (GNNs) to adversarial poisoning attacks on node classification tasks.
Current defensive methods require substituting the original GNNs with defense models, regardless of the original's type.
This approach, while targeting adversarial robustness, compromises the enhancements developed in prior research to boost GNNs' practical performance.
Here we introduce \textsc{Grimm}, the first plug-and-play defense model.
With just a minimal interface requirement for extracting features from any layer of the protected GNNs, \textsc{Grimm} is thus enabled to seamlessly rectify perturbations.
Specifically, we utilize the feature trajectories (FTs) generated by GNNs, as they evolve through epochs, to reflect the training status of the networks. We then theoretically prove that the FTs of victim nodes will inevitably exhibit discriminable anomalies. Consequently, inspired by the natural parallelism between the biological nervous and immune systems, we construct \textsc{Grimm}, a comprehensive artificial immune system for GNNs. \textsc{Grimm} not only detects abnormal FTs and rectifies adversarial edges during training but also operates efficiently in parallel, thereby mirroring the concurrent functionalities of its biological counterparts.
We experimentally confirm that \textsc{Grimm} offers four empirically validated advantages: 1) \emph{Harmlessness}, as it does not actively interfere with GNN training; 2) \emph{Parallelism}, ensuring monitoring, detection, and rectification functions operate independently of the GNN training process; 3) \emph{Generalizability}, demonstrating compatibility with mainstream GNNs such as GCN, GAT, and GraphSAGE; and 4) \emph{Transferability}, as the detectors for abnormal FTs can be efficiently transferred across different systems for one-step rectification.
\end{abstract}

%

\section{Introduction}\label{sec:introduction}

Graph neural networks (GNNs), benefitting from the message passing (MP) strategy, have achieved remarkable success in node classification tasks~\cite{34}. However, GNNs are easily poisoned by imperceptible adversarial perturbations (i.e., inserted/deleted edges) on graph structure during their training phase~\cite{li2020deeprobust,zheng2021graph,wang2019attacking,xi2021graph,meng2023devil,wu2022linkteller}.
Even slight but deliberate perturbations~\footnote{We prioritize defense against graph structure attacks due to their heightened destructiveness compared to feature attacks.} introduced to the training set can poison the target GNN, and then rewire the MP pattern driven by the poisoned GNN, to further misclassify nodes to the target categories~\cite{DaiLT18ADV_GRA,DanielAS2018ADV_GRA}. This may lead to critical issues in many application areas~\cite{72,73,74,75}, including those where perturbations undermine public trust~\cite{Kreps2020SinceAdv}, interfere with human decision making~\cite{Walt2019NMI}, and affect human health and livelihoods~\cite{Samuel2019Sci}. More seriously, it's proved that non-robust GNNs are inevitably poisoned once adversaries take them as the target~\cite{Liu_2022_TNNLS}.

A plethora of robustness models aim to enhance protected GNNs with add-on functions to defend against edge-perturbing attacks, which causes the unavoidable dismissal of the original inner function of the non-robustness GNNs. As representative examples:
(1) RGCN~\cite{Zhu2019RGCN} replaces the hidden representations of nodes in each graph convolutional network (GCN)~\cite{Kpif_2017_ICLR} layer to the Gaussian distributions, to further absorb the effects of adversarial changes.
(2) GCN-SVD~\cite{entezari2020SVD} combines a singular value decomposition (SVD) filter prior to GCN to eliminate adversarial edges in the training set.
(3) STABLE~\cite{li2022STABLE} reforms the forward propagation of GCN by adding functions that randomly recover the roughly removed edges.
(4) EGNN~\cite{liu2021EGNN} leverages graph smoothing techniques based on transductive model PPNP \& APPNP~\cite{gasteiger2018predict}, to confine the permutation setting space, effectively excluding the majority of non-smooth permutations.
(5) GRN~\cite{liu2024towards} incorporates local signals into the central node’s representation to improve resonance-based robustness.

Unfortunately, deploying these defense models requires (partially or completely) replacing the original GNN, inevitably leading to the loss of its custom functions and features.
In practical applications, diverse GNNs, such as GCN, graph attention network (GAT)~\cite{velivckovic_2017_GAT}, and GraphSAGE~\cite{hamilton_2017_SAGE}, are designed for various tasks according to differentiated requirements.
For instance, a node classification task in large graphs relies on inductive frameworks such as GraphSAGE, instead of transductive frameworks such as GCN. However, GCN-based defending models discards the generalizability of inductive frameworks, leading to excessive time spend. This approach sacrifices the efforts made by previous research to enhance GNNs' practical performance.

The examination of the issue necessitates the integration of a non-disruptive, plug-and-play defensive mechanism. The core challenge lies in identifying a describable anomaly pattern that is directly related to the data and independent of the model architecture.

For this challenge, we theoretically demonstrate that the manifestation of anomalies in the MP of a compromised GNN is discernible. This is substantiated by the observable differences in the \emph{feature trajectories} (FTs) of attacked versus non-attacked nodes throughout the training process. These trajectories provide a detectable interface for the implementation of external defensive strategies.

However, formulating anomalous trajectories of node features is possible, yet capturing these observable antibody behaviors remains challenging attributes to two primary issues: 1) Abnormal trajectories are aggregated by illegal messages that are hide inside GNNs and pass along the edges of the graph throughout the training epochs. Due to the non-Euclidean nature of graphs, monitoring fine-grained (including node- and edge-grained) MP and converting them into computable Euclidean tensors is almost intractable, much less capturing these illegal messages. This problem is even more pronounced in some transductive GNNs (such as GCN). 2) Sufficient samples of illegal messages can only be obtained after the GNN is already poisoned, and these samples are non-transferable due to the randomness of the adversaries, resulting in the inability of the classifier to predict the universal pattern of illegal messages.

In addressing this challenge, we have identified a biologically inspired method. Our research indicates that mimicking the mechanism by which the human immune system (HIS) detects viruses~\cite{kipnis2016multifaceted} serves as an efficacious strategy for the issues delineated previously. This efficacy is ascribed to the synergistic interaction between the HIS and the human nervous system (HNS). The HNS, serving as a biological prototype for neural networks~\cite{morton2020representations}, collaborates with the HIS to safeguard the human body against invasive viral antibodies~\cite{brodin2017human}, without inducing mutual harm~\cite{kenney2014autonomic}.


Therefore, it is intuitive that capturing the distributed illegal messages and rectify them in parallel by imitating the mechanism of HIS.
Here we propose \textsc{Grimm} which inherits the merits of the HIS:
\textit{Harmless}. \textsc{Grimm} does not actively intervene with the inner function of the protected GNN. It is a plug-and-play system in practical applications.
\textit{Parallel}. \textsc{Grimm} detects adversarial edges and rectifies the perturbed graph parallel with the protected GNN during its training phase.
\textit{Generalizable}. \textsc{Grimm} can cooperate with the mainstream GNNs such as GCN, GAT and GraphSAGE while fully maintaining their original inner functions.
\textit{Transferable}. \textsc{Grimm} can improve its defending ability by inter-system information from another system.


In practical application, initially, \textsc{Grimm} is integrated into any layer of the protected GNNs before training begins. It computes the feature trajectories (FTs) of nodes and edges in real-time from this layer's outputs, transforming global message passing in non-Euclidean geometry into computable tensors within Euclidean space. This enables global monitoring of the target. \textsc{Grimm} then starts training simultaneously with the target GNN. Given the hyperparameters of the target GNN, \textsc{Grimm} systematically generates potential detectors from feasible FTs. It evolves these FTs into detectors by selecting a subset of reliable FTs and inversely detecting anomalies within the target. In the final training phase, these detectors identify adversarial edges and correct perturbations in the database in real-time. This process runs concurrently with GNN training, ensuring a clean dataset and a well-trained GNN at completion.

Our contributions are:
\begin{itemize}
  \item To the best of our knowledge, we propose the first plug-and-play defense model against poisoning attacks.
  \item We theoretically demonstrate that nodes under attack form discriminable FTs.
  \item We reconstruct the implementation of the HIS specifically for graph learning scenarios.
  \item We systematically evaluate our proposed defense model on real-world datasets.
\end{itemize}

%
%
%

 \begin{figure*}[htb]
    \centering
    \includegraphics[width=0.98\textwidth]{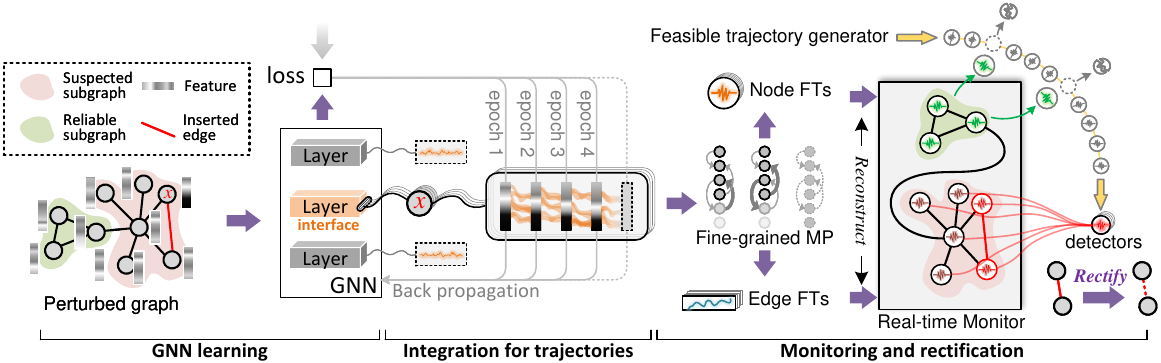}
    \caption{The general workflow of \textsc{Grimm}.}
    \label{fig_overallFrame}
\end{figure*}

\section{Preliminaries}

\subsubsection{Message Passing GNN}

We consider connected graphs $\mathcal{G}=(\mathcal{V},\mathcal{E})$ consisting $N=|\mathcal{V}|$ nodes, where $\mathcal{E}$ is the set of edges.
Let $\mathbf{A}\in \{0,1\}^{N \times N}$ be the adjacency matrix.
Let generic symbol $\mathbf{L}$ be the Laplacian in its broadest sense.
Denote the GNN's number of layers as $L$ and the feature dimension of its $\ell^{\text{th}}$ layer output as $d_\ell$.
The feature and one-hot label matrix are $\mathbf{Z} \in \mathbb{R}^{N\times d_0}$ and $\mathbf{Y} \in \mathbb{R}^{N\times d_L}$ respectively.
The edge connected nodes $v_i$ and $v_j$ is written as $(v_i, v_j)$ or $(v_j, v_i)$.
The neighborhood $\mathcal{N}_i$ of a node $v_i$ consists of all nodes $v_j$ for which $(v_i, v_j) \in \mathcal{E}$.
Let $\mathrm{deg}_i$ be the degree of node $v_i$.
The feature vector and one-hot label of node $v_i$ are $\mathbf{z}_i$ and $\mathbf{y}_i$.

\subsubsection{Feature Trajectories (FTs) Formed by MP}
GNNs are vulnerable to poisoning attacks that subtly introduce perturbations, manifesting progressively across training epochs. This gradual emergence suggests that oscillatory trends in node signals across GNN layers might reveal latent adversarial edges. Our analysis concentrates on the evolution of output features at layer $\ell$, represented by $\mathbf{z}_{i,\ell}$ for each node $i$ in $\mathcal{V}$. To capture this evolution, we incorporate a temporal dimension, denoting the feature of node $i$ at the $t$-th epoch at layer $\ell$ as $\mathbf{z}_{i,\ell}^{(t)}$ and the collective features at this layer as $\mathbf{Z}_{\ell}^{(t)}$. Initially, $\mathbf{z}_{\cdot,\ell}^{(0)}$ and $\mathbf{Z}_{\ell}^{(0)}$ capture the baseline state of features. The trajectory of $\mathbf{z}_{i,\ell}^{(t)}$ in $\mathbb{R}^{d_\ell}$ is described by:
\begin{equation}\label{eq_traj_form}
\mathcal{T}_{i,\ell} = [\mathbf{z}_{i,\ell}^{(0)}, \mathbf{z}_{i,\ell}^{(1)}, \ldots].
\end{equation}
Section~\ref{sec_THM1} discusses how deviations in $\mathcal{T}_{i,\ell}$ indicate attacks, aiding in corrective measures.

Feature aggregation through edges is pivotal in GNN learning. During the message passing from epoch $t$ to $t+1$, the feature quantity transmitted between nodes $i$ and $j$ affects $\mathbf{z}_{i,\ell}^{(t)}$ and $\mathbf{z}_{j,\ell}^{(t)}$, although each node's feature also reflects contributions from all adjacent edges. For edge $(i, j)$, the features transmitted at epoch $t$ at layer $\ell$ are denoted as $\mathbf{z}_{(i,j),\ell}^{(t)}$ or $\mathbf{z}_{(j,i),\ell}^{(t)}$, indicating bidirectional feature flow in undirected graphs. The trajectory of features on the edge $(i, j)$ at layer $\ell$ evolves as:
\begin{equation}\label{eq_trajE_form}
  \mathcal{T}_{(i,j),\ell} = [\mathbf{z}_{(i,j),\ell}^{(0)}, \mathbf{z}_{(i,j),\ell}^{(1)}, \ldots].
\end{equation}
The specific method for obtaining FTs is detailed in the Appendix~\ref{appen_traj_1}.

\subsubsection{Mechanism of Identifying Viruses in the HIS}

The mechanism for virus identification in the HIS is efficiently replicated in the domain of artificial immune systems (AIS)~\cite{dasgupta2006advances} for viruses detection. The core processes of AIS and the corresponding elements in \textsc{Grimm} (indicated within parentheses) are:
1) Creation of \ul{Immune Cells} (\textbf{feasible FTs}): Immune cells are generated indiscriminately to counter all possible antigens.
2) Elimination of \ul{Redundant Cells} (\textbf{reliable FTs}): Immune cells responsive to benign antigens are eradicated.
3) Virus Detection: \ul{Active immune cells} (\textbf{detectors}) undertake the detection of viruses.
4) \ul{Vaccine} (\textbf{transferred detectors}) Production: Functional immune cells synthesize vaccines, which are then disseminated across various human immune systems for broader protection.

\emph{Note that AIS primarily as a conceptual tool to integrate HIS strategies rather than as a direct implementation method.}

\section{The Proposed Model}


%

\subsection{Overview}\label{sec_overview}

The primary workflow of \textsc{Grimm} encompasses 3 stages:

$\blacktriangleright$ \textbf{Integrate FTs}: This involves interfacing with a certain layer of the protected GNN to monitor the feature trajectories of all nodes and edges at that layer. Specifically, for a given layer $\ell$ and nodes $i, j$, it aims to acquire real-time updated $\mathcal{T}_{i,\ell}$ and $\mathcal{T}_{(i,j),\ell}$.

$\blacktriangleright$ \textbf{Detect abnormal FTs}: This phase is dedicated to identifying abnormal edge FTs. The task is formalized as finding a classifier:
  \begin{equation}\label{eq_binary}
    \mathcal{I}_{edge}: \{\mathcal{T}_{(i,j),\ell} : \forall (i,j) \in \mathcal{E} \} \to \{\text{normal}, \text{abnormal}\}
  \end{equation}
Edges classified as ``abnormal'' are suspected of perturbations.

$\blacktriangleright$ \textbf{Rectify perturbations}: Post the detection of abnormalities, the identified perturbations are rectified upon, informing further rectifications.

%

Subsequent sections will elaborate on the specific implementations of the different components.

\begin{figure}[b]
\centering                                                                                                                                                                                                                                                                                                                                                                                                                             \includegraphics[width=0.46\textwidth]{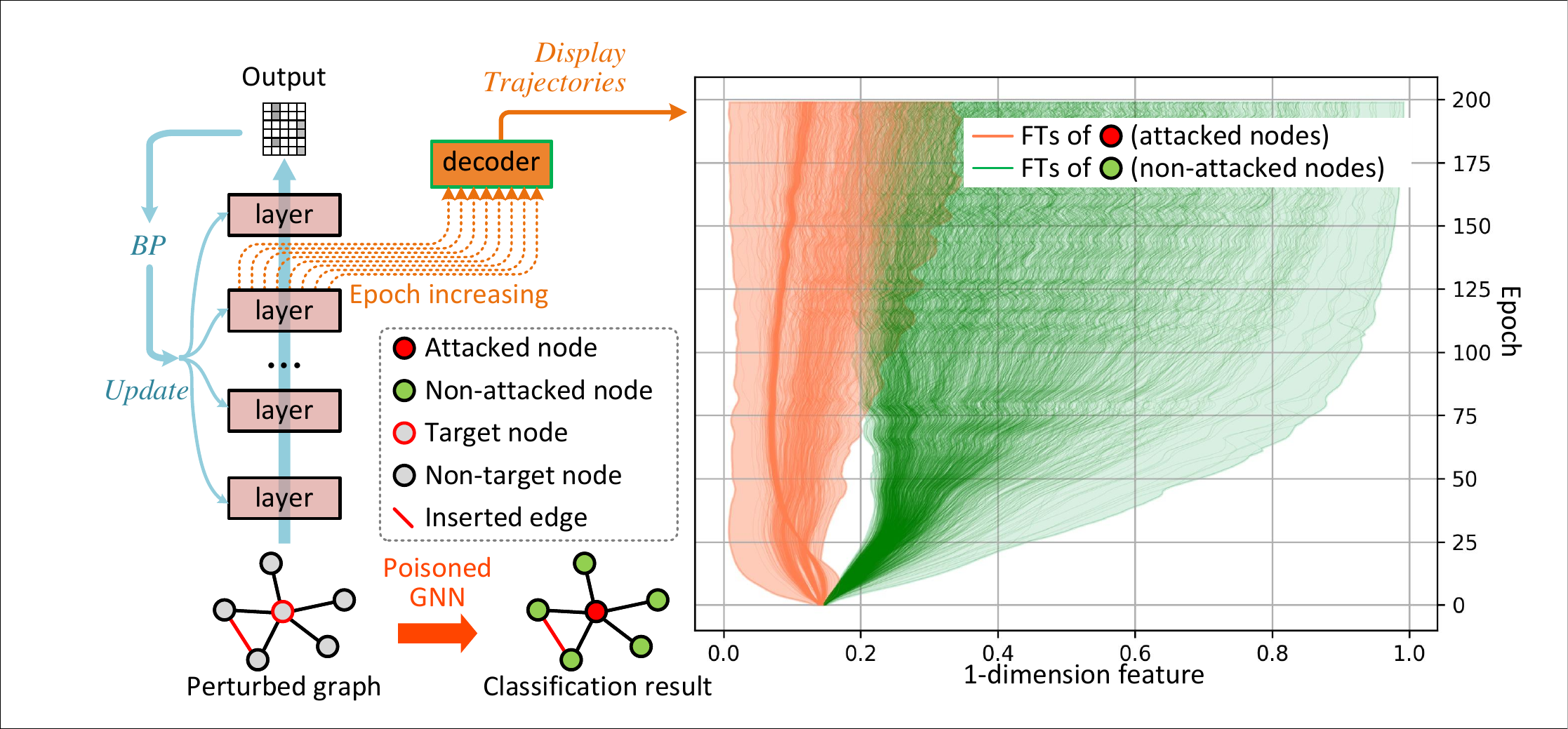}
\caption{FTs of attacked vs. non-attacked nodes.}
\label{fig_motive}
\end{figure}

\subsubsection{Motivation}\label{sec_motiv}

During the training phase of GNN, node representations dynamically evolve, forming trajectories in an equidimensional feature space. We hypothesize that these trajectories can reveal adversarial entities through the observed behaviors of antibodies. To investigate, we replicate a poisoning attack on a GNN, employing the Metattack methodology~\cite{Zuger2018METTACK} on a 4-layer GCN with the Cora dataset. We introduce a frozen decoder (mapping from $\mathbb{R}^{d_i}$ to $\mathbb{R}$) to the penultimate hidden layer of the GCN to capture one-dimensional features of this layer.

The decoder distinguishes the trajectories of nodes under attack (maliciously misclassified) from those not targeted (correctly classified) during the GCN's training. These trajectories, along with the methodology and results, are presented in Figure~\ref{fig_motive}. It is evident that adversarial edges induce distinctive trajectory patterns in their adjacent nodes, unlike non-attacked nodes. This observation underscores the potential of AIS in detecting and intercepting these illicit communications within the network.

\subsection{Theoretical Foundations}\label{sec_theofound}

\subsubsection{FTs' Discriminability} \label{sec_THM1}
We first fortify the observations noted in Section~\ref{sec_motiv} (Motivation) with a theoretical foundation through the following theorem:
\begin{theorem}\label{thm_dis}
Consider a GNN undergoing a poisoning attack. Let $\mathcal{V}_{adv}$ and $\mathcal{V}_{non}$ respectively denote the sets of nodes with categories that have been compromised and those that remain uncompromised, The following classification function exists for all layer $\ell$:
\begin{equation}\label{eq_thm_discri}\small
  \mathcal{I}_{node}:  \mathcal{T}_{i,\ell} \to \{\mathcal{T}_{j,\ell} , \mathcal{T}_{k,\ell} \},
  \text{ s.t.: } \forall i \in \mathcal{V}, \forall j \in \mathcal{V}_{adv},  \forall k \in \mathcal{V}_{non}.
\end{equation}
That is, FTs of attacked and non-attacked nodes are discriminable.
\end{theorem}

Proof in Appendix~\ref{appen_1}. This conclusion unveils an inevitable consequence of the attack: once an attack occurs, the FTs inevitably reveal the malicious activity. This finding lays the theoretical groundwork for detecting abnormal FTs.

\subsubsection{Lower bound of FTs' inner products}

Given the operation of the MP process under $\mathcal{M}(\cdot)$ with predefined hyperparameters such as the learning rate $\eta$, the fluctuation in node signal magnitudes is inherently constrained. This constraint ensures predictability in the oscillation range of node signals across $\mathcal{E}$, regardless of edge distribution.

Drawing on insights from Theorem~\ref{thm_dis}, we observe that signal trajectories tend to stabilize during training. This stability is demonstrated by the limited spatial angle between direction vectors across successive epochs, implying that the inner product of these vectors meets a minimum threshold. Specifically, for node $i$, consider:
\begin{align}\label{eq_gamma}
 \gamma^{\text{node}}_{i,\ell} & = (\mathbf{z}_{i,\ell}^{(t+1)} - \mathbf{z}_{i,\ell}^{(t)}) \cdot (\mathbf{z}_{i,\ell}^{(t+2)} - \mathbf{z}_{i,\ell}^{(t-1)}), \notag \\
 \gamma^{\text{edge}}_{(i,\cdot),\ell} & = (\mathbf{z}_{(i,\cdot),\ell}^{(t+1)} - \mathbf{z}_{(i,\cdot),\ell}^{(t)}) \cdot (\mathbf{z}_{(i,\cdot),\ell}^{(t+2)} - \mathbf{z}_{(i,\cdot),\ell}^{(t-1)}),
\end{align}
we have the following proposition.

\begin{proposition}\label{pro_upperbound}
Let {\small $\mathcal{P}(\mathbf{Z}_{\ell-1};\mathcal{E})=\mathbf{L} (\mathbf{L} {\mathbf{Z}_{\ell-1}} \mathbf{Z}_{\ell-1}^{\top})^{\top})^{\top}$ } be a  modular MP model, $\mathbf{J}$ be ones matrix, and $(\cdot)^{\circ -1} $ be Hadamard inverse.
At the end of the $t^{\text{th}}$ epoch, for any node $i$ and layer $\ell$, we have
{\small
\begin{multline}
\gamma^{\text{node}}_{i,\ell} \geq \lambda \quad \text{and} \quad \gamma^{\text{edge}}_{(i,\cdot),\ell} \geq \lambda, \text{s.t.,} \\
 \lambda = \eta^2 \max_k \| ((\mathcal{P}(\mathbf{Z}_{\ell-1};\mathcal{E})) ( (\mathbf{J} - \mathbf{L} \mathbf{Z}_{\ell-1} \mathbf{W}_\ell^{(t)})^{\circ -1} -\mathbf{Y}))_{k,\cdot} \|^2_2,
\end{multline}
}
\end{proposition}

Proof in Appendix~\ref{appen_2}. Note that we only gives the mathematical bound for GCN since its a non-intuitive transductive message passing model. Bound of GAT and GraphSAGE are assigned as a empirical constant.

\subsection{Detecting Abnormal FTs and Rectifying Perturbations}

The essence of detecting abnormal FTs lies in obtaining detectors for these abnormal FTs based solely on reliable FTs. This process unfolds in two steps:

1) Under the constraints set forth in Proposition~\ref{pro_upperbound}, exhaustively enumerate all feasible FTs. 2) Eliminate those FTs that exhibit a relatively small mean squared error (MSE) in comparison to the reliable FTs. The remaining entities constitute effective detectors.

\subsubsection{A Cyclic Self-Supervised Generator of FTs}\label{sec_generate}

\begin{figure}[h]
\centering
\includegraphics[width=0.45\textwidth]{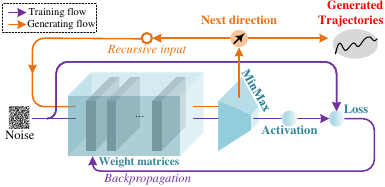}
\caption{The general workflow of the trajectory generator.}
\label{fig_TGen}
\end{figure}

Feasible trajectories are constrained by Proposition~\ref{pro_upperbound}. The FT generator creates valid trajectories. However, pre-training trajectory generation is impractical due to dynamic adversarial intervention in MP of GNNs. Thus, synchronous real-time generation is essential as training epochs progress, as illustrated in Figure~\ref{fig_TGen}.

Figure\ref{fig_TGen} outlines two processes: training and generation, indicating the generator's self-training and cyclic real-time direction vector output based on the previous epoch's vector. Specifically, for a valid vector $v$, at layer $p$ and epoch $e$, the optimal generator $G(\cdot;\theta)$ with trainable weights $\theta$ executes
$G(v;\theta) \cdot v \leq \gamma^{\text{node}}_i$.
This requires first training $G(\cdot)$ by addressing the self-supervised loss derived from Proposition\ref{pro_upperbound}, i.e.,
{\small$ \min_{\theta} \left(\mathrm{Sigmoid}(G(noise;\theta) \cdot noise) - \gamma^{\text{node}}_i\right)$},
where $noise$ is randomly sampled noise. Then, for a target trajectory length $\varrho$, the trained generator yields a vector set:
\begin{equation}\label{eq_gen_vector}\small
[\mathcal{T}^{\text{init}}, G(\mathcal{T}^{\text{init}};\theta), G(G(\mathcal{T}^{\text{init}};\theta);\theta), \ldots],
\end{equation}
where $\mathcal{T}^{\text{init}}$ is the initial direction vector. Finally, detectors are reconstituted by generated vector sets. The generator continues to produce feasible trajectories until sufficient, i.e., they almost span the entire feasible domain of FTs.

\subsubsection{Producing detectors Using Negative Selection Algorithm (NSA)}\label{sec_producemature}

The NSA posits that the elimination of normal samples from all feasible samples results in a detector $\mathcal{T}_{det}$ for abnormal samples. We have demonstrated the feasibility of this approach for FT samples in Section~\ref{sec_theofound} (Theoretical Foundations). In other words, by simply discarding those feasible FTs generated by the generator, which have a MSE with reliable FTs less than a given threshold $\rho$ what remains are effective detectors for abnormal FTs. This method is straightforward yet intuitive.
In the training phase of GNNs, the dimensionality of FTs increases as the number of training epochs $t$ increases. Hence, we define checkpoints; detectors are generated when the epoch reaches this point. Between two checkpoints, we only monitor the GNN without taking specific actions, until the epoch reaches the next checkpoint. The set of detector is denoted as $\mathbb{T}$.

\subsubsection{Detecting}\label{sec_nonselfidentify}

\textbf{Detecting abnormal FTs}.
The determination of whether a FT is abnormal is contingent upon its average MSE in relation to all detectors. If this average MSE falls below the threshold $\rho$, then the FT in question is classified as an abnormal FT. In other words, given the layer $\ell$, the function $\mathcal{I}_{node}$, as delineated in Theorem~\ref{thm_dis}, can be instantiated as
{\small
\begin{align}\label{eq_detect_FT}
\forall i\in \mathcal{V}, \mathcal{T}_{det} \in \mathbb{T},
\mathcal{I}_{node}(\mathcal{T}_{i,\ell}) =
\begin{cases}
  \mathcal{T}_{adv}, & \mbox{} MSE(\mathcal{T}_{i,\ell}, \mathcal{T}_{det}) \leq \rho \\
  \mathcal{T}_{non}, & \mbox{otherwise},
\end{cases}
\end{align}
where $\mathcal{T}_{adv}$ and $\mathcal{T}_{non}$ are FTs formed on the attacked and non-attacked nodes.
}
\textbf{Detecting inserted edges}. As adversarial edges pass illegal messages which may mislead the classification of the target nodes, the node FTs on their directly connected node will act abnormally. Therefore, we can identify inserted edges according to the abnormal node FTs. As shown in Figure~\ref{fig_DetectionTotal}.(a), for each $\mathcal{T}_{i}$, we query the trajectories $\mathcal{T}_{i,(i,j)}$ of edges which directly connect to node $i$ where $(i,j) \in \mathcal{E}$. If any abnormal trajectories $\mathcal{T}_{i,(i,k)}$ exist, the edge $(i,k)$ is detected as the inserted edge.
\textbf{Detecting deleted edges}. Similar to the aforementioned analysis, if all trajectories of edges connect to node $i$ are normal, while the FT of node $i$ is abnormal, we can confirm that one of the edges should be connected to node $i$ is deleted. Then, as shown in Figure~\ref{fig_DetectionTotal}.(b), aiming at locating the deleted edge $(i,o)$, $\mathcal{G}'$ can be rectified circularly until the trajectory on $(i,o)$ is identified as normal FT, i.e., if $\mathcal{T}_{i,(i,o)}$ is abnormal, edge $(i,o)$ is detected as the deleted edge.
\begin{figure}[htb]
\centering
\includegraphics[width=0.47\textwidth]{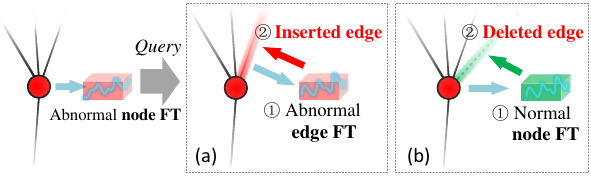}
\caption{The detection method for adversarial edges.}
\label{fig_DetectionTotal}
\end{figure}

\subsubsection{Rectifying perturbations}\label{sec_repair}
Once the adversarial edge is detected, \textsc{Grimm} rectifies the perturbed graph and meanwhile does not terminate the training process. However, the model is still poisoned to a certain extent despite some adversarial edges on the graph are rectified immediately. Therefore, after rectifying, we rollback the version of the trainable matrix $W$ to the previous $\delta$ epochs, to thus reduce the harm caused by the rectified adversarial edges. For instance, if an edge is detected as the inserted edge in epoch $t$, the perturbed graph $\mathcal{G}'$ is rectified as $\mathcal{G}'_{r,t}=\{\mathbf{Z},\mathcal{E}'_{r,t}\}$ by deleting the corresponding edge. In the experiments, we find that rolling back of $\mathbf{W}$ will cause drop in accuracy during the training phase. Then, the limited lower accuracy caused by the adversarial edges will be broken through after several epochs. Pseudo-code of \textsc{Grimm} is provided in Appendix~\ref{sec_pseudo}.

\section{Experiments}

The effectiveness of \textsc{Grimm} is evaluated under multiple aspects including: 1) global accuracy, 2) harmlessness, 3) transferability, 4) interface position, 5) sensitivity of MSE's threshold, and 6) runtime comparison of robust functions, and 7) internal rectification details (\emph{c.f.} Appendix~\ref{sec_app_exp}).

\textbf{Datasets.} Our approaches are evaluated on six real-world datasets widely used for studying graph adversarial attacks~\cite{Liu_2022_TNNLS,liu2024towards}. These datasets are two types:
\textit{Small graphs}, which include the citation datasets Cora and Citeseer, as well as Polblogs, representing social networking data.
\textit{Large graphs}, which include \emph{Brain}\cite{wang2017Braindataset}, \emph{Pubmed}, and \emph{Reddit}\cite{hamilton_2017_SAGE}.
\textit{For each dataset, we randomly partitioned 1/10th of the region to serve as a reliable subgraph}, ensuring the absence of perturbations within this designated area.
%
%


\begin{table*}[htbp]
  \centering \small
\setlength{\tabcolsep}{1.25mm}{
    \begin{tabular}{c | c | ccc | ccccc ccc ccc}
    \toprule
        \multirow{2}[0]{*}{{Dataset}} & \multirow{2}[0]{*}{{$p_r$}} & \multicolumn{3}{c |}{Unprotected models} & \multicolumn{5}{c }{Non-symbiotic defending models} & \multicolumn{3}{c }{Protected by \textsc{Grimm}} & \multicolumn{3}{c}{Protected by \textsc{Grimm} (E)} \\
    \cmidrule(lr){3-5} \cmidrule(lr){6-10} \cmidrule(lr){11-13} \cmidrule(lr){14-16}
        &     & GCN & GAT & SAGE & RGCN & SVD & Pro & Jaccard & EGNN & GCN & GAT & SAGE & GCN & GAT & SAGE  \\
    \cmidrule(lr){1-2}\cmidrule(lr){3-5} \cmidrule(lr){6-10} \cmidrule(lr){11-13}  \cmidrule(lr){14-16}
        \multirow{1}[0]{*}{{Cora}} & 20  & 59.02 & 59.13 & 63.91 & 59.01 & 56.37 & 64.38 & 73.24 & 69.52 & 73.49 & 73.84 & \textbf{79.51} & 71.02 & 65.70 & 75.40 \\
        \multirow{1}[0]{*}{{Citeseer}} & 20  & 62.73 & 60.14 & 68.81 & 62.90 & 57.65 & 55.54 & 66.92 & 65.61 & 70.76 & 72.42 & \textbf{76.91} & 68.42 & 66.01 & 75.48 \\
        \multirow{1}[0]{*}{{Pubmed}} & 20  & 70.02 & 69.89 & 72.09 & 70.88 & 81.54 & 82.57 & 76.61 & 78.91 & 75.10 & \textbf{79.30} & 74.05 & 73.39 & 78.04 & 76.89 \\
        \multirow{1}[0]{*}{{Polblogs}} & 20  & 52.33 & 50.28 & 54.12 & 58.04 & 55.41 & 73.60 & 70.55 & 75.95 & \textbf{82.11} & 81.02 & 78.62 & 76.47 & 76.44 & 75.40 \\
    \bottomrule
    \end{tabular}
    }
   \caption{Classification accuracy (\%) on the attacked graph after rectifying. \textsc{Grimm} (E) means producing detectors based on the exogenous reliable FTs. SVD, Pro, and SAGE is the abbreviation of GNN-SVD, Pro- GNN and GraphSAGE.}
  \label{tab:GlobalAcc}%
\end{table*}%

\textbf{Baselines.} The proposed \textsc{Grimm} model not only protects non-defense GNNs against poisoning attacks but also surpasses other defense models. Baseline models are categorized into three groups:

\underline{\textit{Non-defense GNNs.}} To illustrate \textsc{Grimm}'s efficacy against edge-perturbing attacks, we selected representative GNNs:
   1) \emph{GCN}, a prevalent architecture,
   2) \emph{GAT}, a typical GNN variant,
   3) \emph{GraphSAGE}, effective in aggregating spatial features.

\underline{\textit{Comparison defending models.}} We benchmark \textsc{Grimm} against defense models:
   1) \emph{RGCN}, using Gaussian distributions to mitigate adversarial impacts,
   2) \emph{GNN-SVD}, employing a truncated SVD for adjacency matrix approximation,
   3) \emph{Pro-GNN}, focusing on intrinsic node properties for robustness,
   4) \emph{Jaccard}~\cite{wu2019jaccard}, based on Jaccard similarity for defense,
   5) \emph{EGNN}, filtering perturbations via $\ell_1$- and $\ell_2$-based graph smoothing.

\underline{\textit{Attack methods.}} Experiments consider attack strategies:
   1) \emph{Metattack}, a meta-learning based approach,
   2) \emph{CLGA}~\cite{zhang2019CLGA}, an unsupervised tactic,
   3) \emph{RL-S2V}~\cite{DaiLT18ADV_GRA}, leveraging reinforcement learning.

All experiments were conducted using an NVIDIA RTX 4080s GPU, an Intel i7-14700-KF CPU, and 64GB of RAM.

\subsubsection{Accuracy After Rectification}\label{sec_Acc_aft_rec}

\begin{figure*}[htb]
\centering
\includegraphics[width=0.95\textwidth]{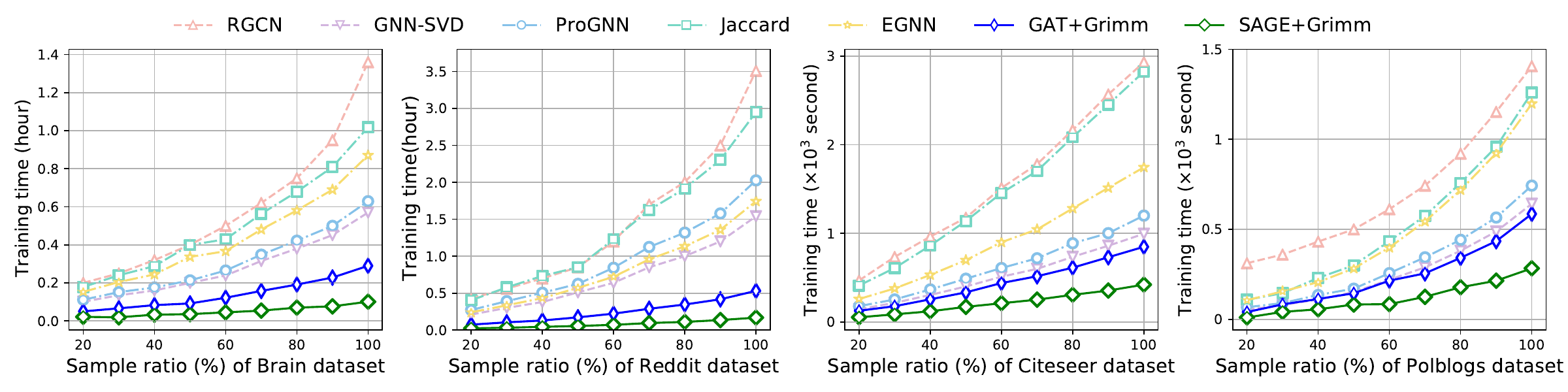}
\caption{Training time under increased sample ratio.}
\label{fig_increasing}
\end{figure*}

Here we evaluate the global classification efficacy of \textsc{Grimm} by applying Metattack to disrupt target models and measuring the post-training classification accuracy, with results documented in Table~\ref{tab:GlobalAcc}. The evaluation captures outcomes for models shielded by \textsc{Grimm} after global rectification and once model loss stabilizes. The mean accuracy and its deviation are presented. Observations indicate that \textsc{Grimm} effectively protects non-defense GNNs under attacks and outperforms leading robust GNNs, with few exceptions:
1) In the Pubmed dataset at $p_r = 10\%$, EGNN, using graph smoothing for enhanced adversarial robustness, handles localized perturbations effectively, but such scenarios are infrequent and \textsc{Grimm}'s accuracy improves as $p_r$ increases.
2) For the Polblogs dataset at $p_r = 0\%$, \textsc{Grimm} slightly lags behind Pro-GNN by 0.32\%. However, as $p_r$ increases, \textsc{Grimm}-protected models show the smallest decline in accuracy among baselines, maintaining a leading position.

In non-adversarial settings ($p_r = 0\%$), rectification aims to boost accuracy by optionally modifying edges, as no changes are needed for a clean graph. When reliable FTs are unattainable due to the absence of a dependable subgraph on $\mathcal{G}'$, we base our experiments on reliable FTs recommended from an exogenously reliable subgraph perturbed at a 20\% rate, achieving high accuracy.

\begin{table*}[h]
  \centering
 \resizebox{\linewidth}{!}{
    \begin{tabular}{c c ccc | cccccc ccc}
    \toprule
    \multirow{2}[0]{*}{Dataset} & \multirow{2}[0]{*}{Attack} & \multicolumn{3}{c |}{Unprotected models} & \multicolumn{6}{c}{Defending models} & \multicolumn{3}{c}{Models protected by \textsc{Grimm}} \\
    \cmidrule(lr){3-5} \cmidrule(lr){6-11} \cmidrule(lr){12-14}
        &     & GCN & GAT & SAGE & RGCN & SVD & Pro & Jaccard & EGNN & \emph{Avg. e. t.} & GCN & GAT & SAGE \\
    \cmidrule(lr){1-2}\cmidrule(lr){3-5} \cmidrule(lr){6-11} \cmidrule(lr){12-14}

    \multirow{3}[0]{*}{Brain} & Metattack & 0.34  & 0.26  & 0.08  & 1.36  & 0.57  & 0.63  & 1.02  & 0.87 & +982.7\%  & 0.38 [+11.3\%]  & 0.29 [+8.4\%]  & \textbf{0.10} [+25.4\%]  \\
        & CLGA & 0.35  & 0.26  & 0.08  & 1.39  & 0.59  & 0.62  & 1.06  & 0.86 & +1009.8\% & 0.39 [+10.8\%]  & 0.29 [+10.4\%]  & \textbf{0.11} [+36.9\%] \\
        & RL-S2V & 0.35  & 0.27  & 0.08  & 1.38  & 0.58  & 0.64  & 1.06  & 0.87 & +1017.2\%  & 0.38 [+8.1\%]  & 0.30 [+11.8\%]  & \textbf{0.10} [+28.1\%]  \\

    \cmidrule(lr){1-2}\cmidrule(lr){3-5} \cmidrule(lr){6-11} \cmidrule(lr){12-14}

    \multirow{3}[0]{*}{Pubmed} & Metattack & 0.53  & 0.31  & 0.10  & 1.73  & 0.59  & 0.66  & 1.26  & 1.17  & +1011.3\% & 0.57 [+7.7\%]  & 0.34 [+10.4\%]  & \textbf{0.12} [+27.7\%]  \\
        & CLGA & 0.54  & 0.30  & 0.09  & 1.72  & 0.61  & 0.67  & 1.28  & 1.22  & +1063.7\%  & 0.59 [+8.5\%]  & 0.35 [+15.0\%]  & \textbf{0.14} [+47.3\%] \\
        & RL-S2V & 0.53  & 0.32  & 0.09  & 1.72  & 0.61  & 0.68  & 1.25  & 1.18  & +1071.3\%  & 0.58 [+8.9\%]  & 0.36 [+14.4\%]  & \textbf{0.13} [+38.5\%]  \\

    \cmidrule(lr){1-2}\cmidrule(lr){3-5} \cmidrule(lr){6-11} \cmidrule(lr){12-14}

    \multirow{3}[0]{*}{Reddit} & Metattack & 1.28  & 0.50  & 0.13  & 3.51  & 1.54  & 2.03  & 2.95  & 1.74  & +1732.6\%  & 1.34 [+4.6\%]   & 0.53 [+6.7\%]  & \textbf{0.17} [+30.3\%] \\
        & CLGA & 1.27  & 0.50  & 0.12  & \textbackslash{} & 1.52  & $\backslash$ & \textbackslash{} & 1.69  & +1192.4\%  & 1.32 [+3.9\%]  & 0.54 [+7.4\%]    & \textbf{0.17} [+33.8\%]  \\
        & RL-S2V & 1.29  & 0.50  & 0.13  &   \textbackslash{}  & 1.50  &  \textbackslash{}   &  \textbackslash{}   & 1.70  & +1167.7\%  & 1.34 [+3.6\%]  & 0.54 [+7.9\%]  & \textbf{0.17} [+35.8\%] \\
    \bottomrule
    \end{tabular}}%
    \caption{Training time (hour) of defending models on large graphs. [+5\%] indicating a 5\% increase in training duration relative to the protected GNN, and \emph{Avg.}\emph{e.}\emph{t.} represents the mean additional runtime in comparison to GraphSAGE.}
    \label{tab:harmless2}
\end{table*}%

\subsubsection{Harmfulness of \textsc{Grimm} and Other Models}\label{sec_compaire_exp}

The impact of a defense model can be quantitatively assessed by its training duration, particularly when transductive models are used where inductive methods are preferable (e.g., large-scale graphs), leading to a significant increase in computational load. The efficiency of \textsc{Grimm} is demonstrated by its relatively shorter training duration compared to other defense models. This was empirically tested across various attack conditions, documenting the training times until optimal performance was achieved, indicated by loss convergence. The additional time burden imposed by \textsc{Grimm}, expressed as a percentage increase over the original training time, is detailed in Tables~\ref{tab:harmless2} and~\ref{tab:harmless1} (\emph{c.f.} Appendix~\ref{sec_app_exp}). On smaller graphs, traditional defense models can extend training up to 3-5 times longer than usual. In contrast, \textsc{Grimm} introduces minimal additional time, a benefit that is even more evident in larger datasets. For example, on the Reddit dataset, existing models like EGNN can increase training times by up to 18 times, whereas \textsc{Grimm} maintains negligible extra time.

Real-world graphs often grow dynamically, making the ability to learn from scale-increasing graphs crucial for GNNs. To highlight \textsc{Grimm}'s efficiency, we conducted experiments on progressively larger samples of the Brain and Reddit datasets, measuring the training time of various defense models. The results, displayed in Figure~\ref{fig_increasing}, show that while traditional defense methods significantly extend training times as datasets expand, \textsc{Grimm} ensures minimal increase, maintaining shorter overall durations across various scales, demonstrating its harmless integration with the protected model.

\subsubsection{Transferability of Detectors}

\textsc{Grimm} can rectify an attacked graph based on detectors transferred from another system, thus significantly reducing the time cost of graph rectification. In this section, we report the transferability of the detectors. We producedetectors by adopting Metattack on GCN, and transfer them to another \textsc{Grimm} which is protecting different GNN models against different attacks. The perturbation rate of Metattack is set to 20\% and the augmentation rate of CLGA is set to 10\%.
The transferred detectors are produced based on a GCN model and Cora dataset under the corresponding attack. The experimental results are shown in Table~\ref{tab_transfer1}.
The observation highlights that without any training, based on the detectors transferred from an external source, \textsc{Grimm} can still detect and effectively rectify perturbations.

%

\begin{table}[htp]\small
  \centering
  \setlength{\tabcolsep}{1.25mm}{
    \begin{tabular}{c|lcccccc}
       \toprule
         \multirow{2}[0]{*}{Attack} & Dataset & \multicolumn{2}{c}{Cora} & \multicolumn{2}{c}{Citeseer} & \multicolumn{2}{c}{Pubmed} \\
         \cmidrule(lr){2-2}\cmidrule(lr){3-4} \cmidrule(lr){5-6} \cmidrule(lr){7-8}
        & Target & GAT & SAGE & GAT & SAGE & GAT & SAGE \\
         \cmidrule(lr){1-1} \cmidrule(lr){2-2}\cmidrule(lr){3-4} \cmidrule(lr){5-6} \cmidrule(lr){7-8}
    \multirow{2}[0]{*}{Meta} & Before & 59.20 & 62.21 & 61.08 & 68.29 & 69.51 & 71.77 \\
        & After & 68.71 & 73.94 & 70.32 & 77.69 & 80.37 & 82.64 \\
     \cmidrule(lr){1-1} \cmidrule(lr){2-2}\cmidrule(lr){3-4} \cmidrule(lr){5-6} \cmidrule(lr){7-8}
    \multirow{2}[0]{*}{CLGA} & Before & 66.39 & 68.27 & 65.40 & 64.93 & 72.58 & 70.19 \\
        & After & 76.34 & 80.85 & 78.42 & 77.39 & 79.65 & 78.52 \\
       \bottomrule
    \end{tabular}}
  \caption{Classification accuracy (\%) of poisoned GNNs (before) and \textsc{Grimm}-protected GNNs (after). \textsc{Grimm} rectifies perturbed graph based on transferred detectors.}
  \label{tab_transfer1}%
\end{table}%

\begin{figure}[htb]
\centering
\includegraphics[width=0.40\textwidth]{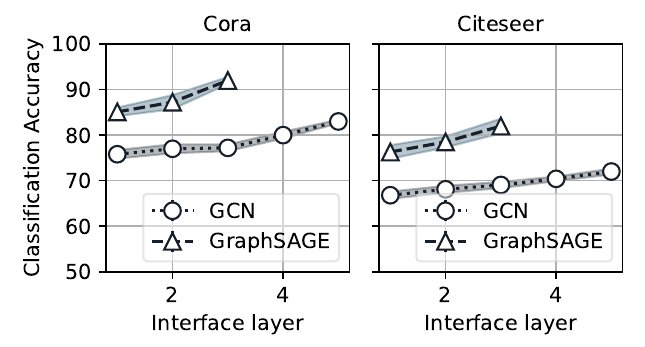}
\caption{Classification accuracy of \textsc{Grimm} after interfacing to the different layers.}
\label{fig_interface}
\end{figure}

\subsubsection{Interface Position for \textsc{Grimm}}
The primary experiments presented in the main text involve \textsc{Grimm} interfacing with the penultimate layer of the protected GNNs. This section explores how the interface position within the GNNs affects global classification accuracy. We illustrate the mean value (central line) and standard deviation (shadow areas) of results, repeated 10 times, in Figure~\ref{fig_interface}.
Figure~\ref{fig_interface} shows a clear trend: classification accuracy progressively improves as the interfaced layer's depth within the GNNs increases. This pattern aligns with the operational dynamics of GNNs, where features are extracted sequentially from graph data across layers. Thus, interfacing deeper within the network allows for the extraction of increasingly precise and sophisticated features. This enhancement in feature representation intrinsically boosts global classification accuracy, confirming the trend that deeper interfacing layers within the GNNs lead to improved performance.

\subsubsection{Sensitivity of MSE's threshold $\rho$}
The process of identifying abnormal FTs in \textsc{Grimm} depends on calculating their MSE) relative to the detectors. The choice of the threshold $\rho$, which sets the MSE boundary for detecting abnormalities, is crucial to \textsc{Grimm}'s performance dynamics. We then focus how accuracy responds to changes in $\rho$ settings.
The results of this investigation are visually presented in Figure~\ref{fig_Sensitivity}. The graph shows that the permissible range for $\rho$ is relatively wide, indicating that \textsc{Grimm} has low sensitivity to fluctuations in $\rho$. This trait suggests that \textsc{Grimm}'s performance remains robust and stable, even with varying $\rho$ values. Such resilience to the MSE threshold not only highlights the robustness of the \textsc{Grimm} framework but also provides operational flexibility, allowing for some latitude in selecting $\rho$ without significantly affecting the accuracy of abnormal FT detection.

\begin{figure}[htb]
\centering
\includegraphics[width=0.46\textwidth]{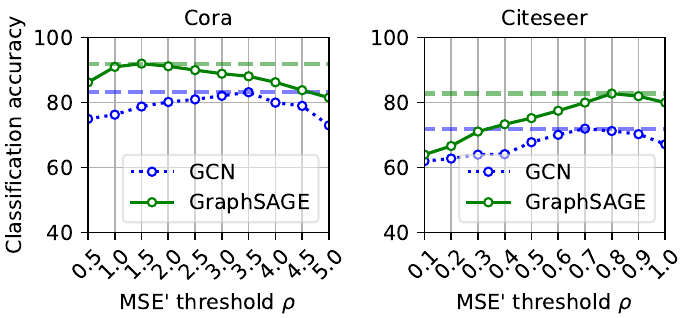}
\caption{Sensibility of the threshold $\rho$ for the affinity.}
\label{fig_Sensitivity}
\end{figure}

\subsubsection{Observations and Discussions}

The findings are summarized as follows:

1) \emph{Vulnerability of GNNs}: Data from Table~\ref{tab:GlobalAcc} highlights the vulnerability of unprotected GNNs (GCN, GAT, GraphSAGE) to perturbations, underscoring the need for robust defense mechanisms.
2) \textsc{Grimm}\emph{'s Superiority}: \textsc{Grimm} outperforms other defense models in mitigating graph perturbations, although minor accuracy reductions are noted in some cases (Table~\ref{tab:GlobalAcc}).
3) \emph{Performance in Non-adversarial Contexts}: In scenarios without adversarial attacks ($P_rate = 0\%$), \textsc{Grimm} improves GNN structure and performance through the strategic addition of auxiliary edges (Table~\ref{tab:GlobalAcc}).
4) \emph{Computational Efficiency}: \textsc{Grimm} significantly reduces operational time compared to baseline models, especially noticeable in larger datasets (Table~\ref{tab:harmless2}, Figure~\ref{fig_increasing}).
5) \emph{Utility of Transferred Detectors}: \textsc{Grimm} effectively employs transferred detectors to rectify perturbed graphs, avoiding the extensive training of new detectors and thereby saving computational resources (Table~\ref{tab_transfer1}).

%
%
%
%
%

\section{Conclusion}
This paper presents the first plug-and-play defense model \textsc{Grimm} against poisoning attacks for GNNs. \textsc{Grimm} seamlessly integrates with various GNNs without disrupting their intrinsic functions, offering a parallel, non-intrusive, and generalizable defense mechanism. Underlying implementations of HIS are migrated to GNN. \textsc{Grimm} can continuously monitor the MP of GNNs, detect adversarial edges and reflect the perturbed graph. Experiments demonstrated \textsc{Grimm} protects mainstream GNNs including GCN, GAT and GraphSAGE from the most powerful attacks while outperforming the state-of-the-art defenses.


\bibliography{ijcai23}

\clearpage

\appendix

\onecolumn

\section{Additional Experimental Results}\label{sec_app_exp}

\begin{figure*}[!h]
    \centering
    \includegraphics[width=0.99\textwidth]{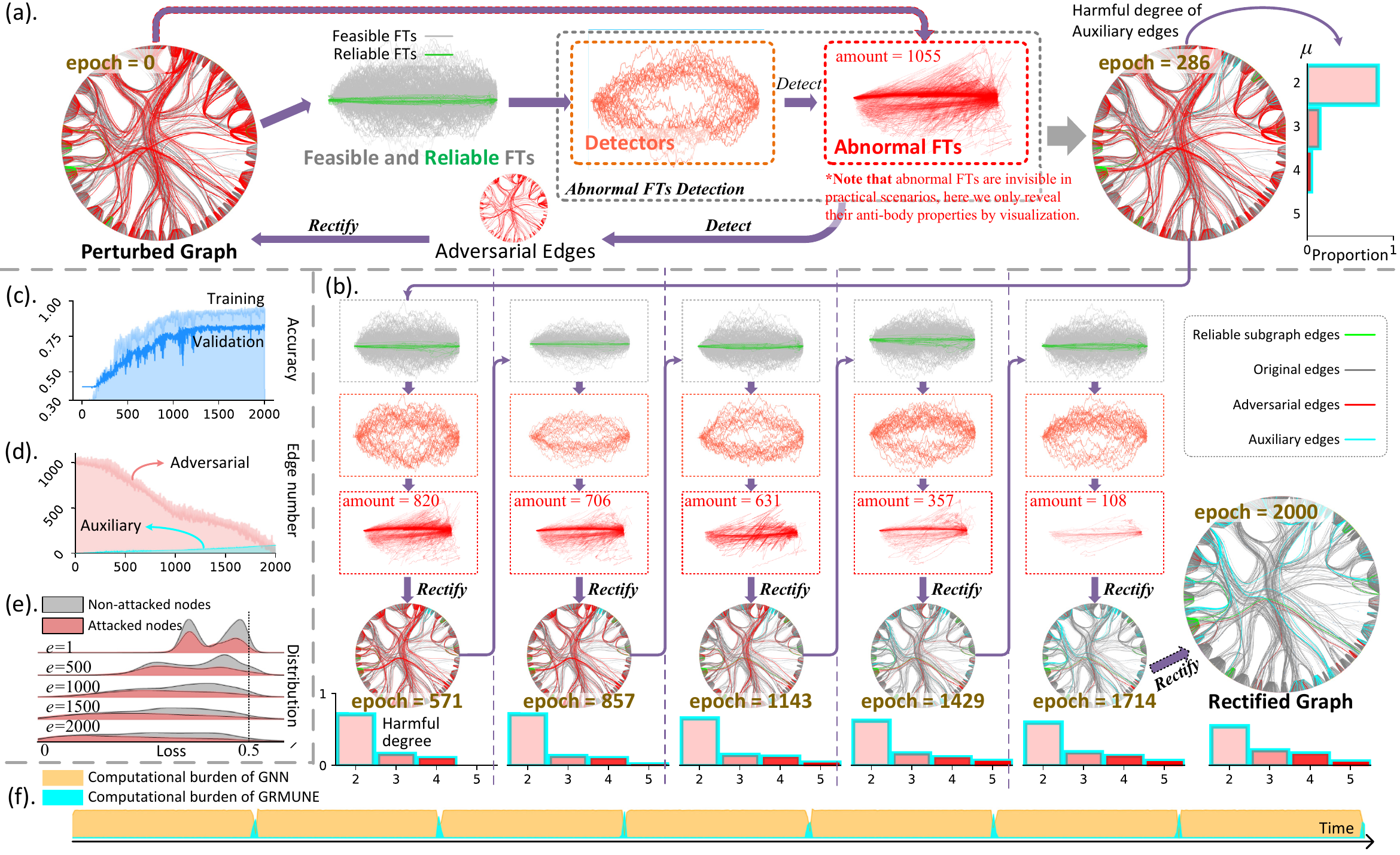}
    \caption{
    \textbf{Internal rectification details of \textsc{Grimm}}.
This experimental endeavor serves to elucidate the progressive graph rectification capabilities of \textsc{Grimm} during the training phase of a GNN, utilizing a demonstrative example for clarity. The experiment employs Metattack as the adversarial methodology, selects the Cora dataset for analysis, and targets a GCN as the model of interest. FTs are extracted from the outputs of the penultimate layer of a 5-layer GCN. We adopt the same approach as in Sections~\ref{sec_motiv} and~\ref{sec_norm} (a frozen decoder) to report 2-dimensional FTs. The encompassing outcomes of this experimental procedure are visually presented in Figure\ref{fig_graphrepairing}.
As delineated in Figure~\ref{fig_graphrepairing}, \textsc{Grimm} demonstrates its proficiency in identifying adversarial edges and in rectifying the attacked graph concurrently with the training progression of the GCN under immunization. Notably, both feasible FTs and detector exhibit efficacy in the detection of abnormal FTs, while imposing only a marginal computational overhead. The experiment's findings are detailed as follows:
\textbf{(a).} Single rectifying operation visualization: This segment, corresponding to the conclusion of epoch 286, displays the visualization of a singular rectification. FTs are categorized and color-coded based on their distinct identities, and the characteristics of adversarial, original, and auxiliary edges within the attacked and rectified graphs are visually delineated to highlight the internal rectification efficacy. The spatial attributes of normal and abnormal FTs within an attacked graph are markedly distinct, a disparity that is effectively captured by detectors, enabling the identification of adversarial edges. The detrimental impact of auxiliary edges is quantified based on the original shortest path, denoted as $\mu$, between indirectly connected nodes via an auxiliary edge; for instance, an auxiliary edge with a $\mu$ value of $2$ is deemed minimally harmful due to the strong structural and feature-based correlation between the connected nodes.
\textbf{(b).} Graph rectification process: This section illustrates the \textsc{Grimm}-mediated graph repair process, highlighting the gradual rectification of adversarial edges post their effective detection by detectors, with the introduction of only nominal auxiliary edges.
\textbf{(c).} Accuracy metrics: The training and validation accuracy rates of the immunized GCN are documented here.
\textbf{(d).} Adversarial and auxiliary edge metrics: The quantity of adversarial edges (inclusive of both inserted and deleted edges) and auxiliary edges are recorded.
\textbf{(e).} Loss distribution ridgeline: This component analyzes the loss distribution ridgeline for nodes misclassified and benign across various epochs. An initial recording of the misclassified nodes is made, attributed to a poisoned GCN trained on $\mathcal{G}'$, followed by the retraining of the same GCN concurrent with the rectification of $\mathcal{G}'$. This process results in the nodes, previously misclassified, being correctly categorized.
\textbf{(f).} Computational load: The computational burden imposed by \textsc{Grimm} and the immunized GCN is outlined, underscoring the parallel efficiency of \textsc{Grimm} with the GCN, without causing significant delays during training.
}
    \label{fig_graphrepairing}
\end{figure*}

\twocolumn

\begin{table*}[htbp]
  \centering
 \resizebox{\linewidth}{!}{
    \begin{tabular}{c c ccc | cccccc ccc}
    \toprule
    \multirow{2}[0]{*}{Dataset} & \multirow{2}[0]{*}{Attack} & \multicolumn{3}{c |}{Unprotected models} & \multicolumn{6}{c}{Defending models} & \multicolumn{3}{c}{Models protected by \textsc{Grimm}} \\
    \cmidrule(lr){3-5} \cmidrule(lr){6-11} \cmidrule(lr){12-14}
        &     & GCN & GAT & SAGE & RGCN & SVD & Pro & Jaccard & EGNN & \emph{Avg. e. t.} & GCN & GAT & SAGE \\
    \cmidrule(lr){1-2}\cmidrule(lr){3-5} \cmidrule(lr){6-11} \cmidrule(lr){12-14}

    \multirow{3}[0]{*}{Cora} & Metattack & 734 & 902 & 277  & 2,140 & 758 & 875 & 1,637 & 1,021 & +363.7\% & 785 [+7.0\%] & 1,022 [+13.3\%] & \textbf{365} [+31.7\%]  \\
        & CLGA & 709 & 892 & 273  & 1,923 & 770 & 836 & 1,409 & 1,195 & +349.3\% & 791 [+11.6\%] & 992 [+11.2\%] & \textbf{355} [+30.0\%]  \\
        & RL-S2V & 711 & 909 & 272  & 2,102 & 765 & 831 & 1,485 & 1,007 & +355.7\% & 760 [+6.9\%] & 1,011 [+11.2\%] & \textbf{362} [+33.1\%]  \\

    \cmidrule(lr){1-2}\cmidrule(lr){3-5} \cmidrule(lr){6-11} \cmidrule(lr){12-14}

    \multirow{3}[0]{*}{Citeseer} & Metattack & 912 & 922 & 303  & 2,932 & 966 & 1,159 & 2,744 & 1,692 & +525.9\% & 993 [+8.9\%] & 1,066 [+16.6\%] & \textbf{405} [+33.6\%]  \\
        & CLGA & 895 & 923 & 305  & 3,016 & 1,014 & 1,165 & 2,568 & 1,570 & +516.7\% & 990 [+10.6\%] & 981 [+6.3\%] & \textbf{401} [+31.5\%]  \\
        & RL-S2V & 927 & 919 & 299  & 3,065 & 984 & 1,147 & 2,933 & 1,565 & +549.2\% & 995 [+6.3\%] & 1,012 [+10.1\%] & \textbf{412} [+37.8\%]   \\

    \cmidrule(lr){1-2}\cmidrule(lr){3-5} \cmidrule(lr){6-11} \cmidrule(lr){12-14}

    \multirow{3}[0]{*}{ Polblogs} & Metattack & 585 & 650 & 202  & 1,407 & 642 & 675 & 1,257 & 1,204 & +412.5\% & 604 [+3.6\%] & 753 [+15.9\%] & \textbf{285} [+41.0\%]  \\
        & CLGA & 614 & 645 & 202  & 1,442 & 651 & 694 & 1,240 & 1,199 & +418.3\% & 658 [+7.2\%] & 738 [+14.4\%] & \textbf{289} [+43.1\%]  \\
        & RL-S2V & 591 & 657 & 204  & 1,439 & 640 & 679 & 1,269 & 1,196 & +412.0\% & 681 [+15.2\%] & 740 [+12.6\%] & \textbf{277} [+35.8\%] \\
    \bottomrule
    \end{tabular}}%
 \caption{Training time (second) of defending models on small graphs. [+5\%] indicating a 5\% increase in training duration relative to the protected GNN, and \emph{Avg.}\emph{e.}\emph{t.} represents the mean additional runtime in comparison to GraphSAGE.}
  \label{tab:harmless1}%
\end{table*}%

Table~\ref{sec_compaire_exp} complements Section~\ref{sec_compaire_exp} (Harmfulness of \textsc{Grimm} and Other Models), presenting results obtained on small graphs. This structured presentation of data ensures a rigorous and detailed evaluation of the models' performance and their impacts in varying experimental contexts.

\begin{figure}[htb]
\centering
\includegraphics[width=0.48\textwidth]{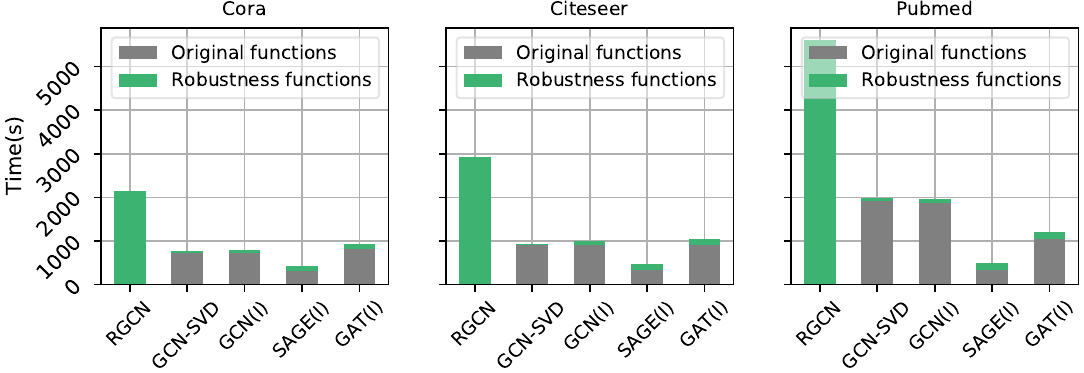}
\caption{ \small \textbf{Runtime comparison of robustness baselines}. This study details the computational durations linked to robust functions within baseline robust GNNs. It assesses the temporal metrics of both original and newly integrated robustness functions, highlighting that the operational latency added by \textsc{Grimm} is minimal.
As discussed in the Introduction, RGCN modifies the forward propagation mechanism typical of conventional GCNs, suggesting that various functionalities of GCN can be seen as robustness enhancements. For GCN-SVD, which incorporates a filtering mechanism within GCN, its robustness primarily stems from this filter. In the \textsc{Grimm} framework, robustness is facilitated through the dynamic processes of anomaly detection and correction. Figure~\ref{fig_runtime} presents these empirical results, with annotations for GCN(I), GAT(I), and SAGE(I) indicating instances of GCN, GAT, and GraphSAGE respectively, reinforced by \textsc{Grimm}.
Comparatively, \textsc{Grimm} introduces only a slight computational overhead compared to GCN-SVD while significantly outperforming RGCN. An in-depth analysis from a system's intrinsic perspective shows that the additional computational load from \textsc{Grimm} on protected GNN architectures is negligible, demonstrating \textsc{Grimm}'s ability to boost GNN robustness efficiently without significant computational costs.}
\label{fig_runtime}
\end{figure}

\section{Pseudo-code of Grimm}\label{sec_pseudo}

\begin{algorithm}[!h]
        \caption{Rectifying process of \textsc{Grimm}}
        \label{alg:GIMPS}
        {\small\begin{algorithmic}[1] 
        \Require Perturbed graph $\mathcal{G}'=\{\mathbf{Z},\mathcal{E}'\}$. The protected GNN $\mathcal{M}$ with the trainable weights $\mathbf{W}$, and the interface layer $\ell$. Maximum epoch $t_{max}$. The checkpoint list for the repairing $t_C$ whose check interval is $c$. Instruction about whether the reliable subgraph on $\mathcal{G}'$ (denoted as $\mathcal{G}_S$) exists. Exogenous reliable graph $\mathcal{G}_{E}=\{\mathbf{Z}_E, \mathcal{E}_E\}$ and a surrogate GNN $\mathcal{M}'$ whose architecture is same as $\mathcal{M}$.
        \For{$t = 1$ to $t_{max}$}
        \State Integrate FTs (Section~\ref{sec_calculate_traj})
        \If {$t \in t_C$}  // \emph{Starting the detection and rectification}
        \If {Reliable subgraph on $\mathcal{G}'$ exists}
        \State \emph{Collect reliable FTs on }$\mathcal{G}_S$
        \Else
        \State \emph{Train} $\mathcal{M}'$ \emph{on} $\mathcal{G}_E$ \emph{for} $c$ \emph{epochs}
        \State \emph{Collect reliable FTs on }$\mathcal{G}_E$
        \EndIf
        \State \emph{Normalize reliable FTs }(Section~\ref{sec_norm})
        \State \emph{Generate feasible FTs }(Section~\ref{sec_generate})
        \State \emph{Produce detectors }(Section~\ref{sec_producemature})
        \State \emph{Detect abnormal FTs }(Section~\ref{sec_nonselfidentify})
        \State \emph{Rectify adversarial edges }(Section~\ref{sec_repair})
        \EndIf
        \EndFor
        \Ensure The rectified graph and the well-trained GNN.
        \end{algorithmic}}
\end{algorithm}

\section{Specific Method for Trajectory Acquisition}\label{appen_traj_1}

\subsection{Mathematical Formulation of Trajectories}

\subsubsection{Dissecting Layer-wise Message Passing in GNNs}
GNNs use a hierarchical architecture that combines node-specific features with overall network topology through layer-wise message passing (MP), defined as $\mathcal{M}: \mathbb{R}^{d_0} \to \mathbb{R}^{d_L}$. This process at layer $\ell$ is given by $\mathbf{Z}_{\ell+1} = \mathcal{M}_\ell(\mathbf{Z}_\ell)$, where $\mathcal{M}_L$ symbolizes the cumulative MP operation across $L$ layers, simplifying to $\mathcal{M}$. This framework is adaptable, supporting different model instantiations. For example, in a transductive GCN model, the MP operation up to the first $\ell$ layers is represented as:
\[
\mathcal{M}_{\ell, \texttt{GCN}}(\mathbf{Z}_{\ell+1}) = \varrho (\mathbf{L} \mathbf{Z}_\ell \mathbf{W}_\ell),
\]
where $\varrho$ denotes the activation function.

In inductive models like GraphSAGE, the MP for a node $i$ up to layer $\ell$ integrates generalized node interactions:
\[
\mathcal{M}_{\ell, \texttt{SAGE}} (\mathbf{z}_{i,\ell+1}) = \varrho (\mathbf{W}_\ell \mathrm{ \scriptstyle CONCAT} (\mathrm{ \scriptstyle AGG}(\{ \mathbf{z}_j : \forall j \in \mathcal{N}(i) \}) ) ).
\]

For Graph Attention Networks (GAT), the corresponding MP process is:
\[
\mathcal{M}_{\ell, \texttt{GAT}} (\mathbf{z}_{i,\ell+1}) = \varrho ( \sum_{j \in \mathcal{N}(i)} a_{i,j} \mathbf{W}_\ell \mathbf{z}_j ).
\]
This paper generally refers to $\mathcal{M}_\ell$ as the standardized MP across the initial $\ell$ layers unless otherwise specified.

\subsubsection{Monitoring GNNs and Integrating FTs}\label{sec_calculate_traj}

Considering a target node $i$ and its adjacent node $j$, we denote the direction vector of the edge FTs along edge $(i,j)$ between epochs $e$ and $e+1$ as
\begin{equation}\label{eq_direction}
 \mathcal{T}_{i,(i,j),\ell}^{(t)\to (t+1)}  = \mathbf{z}_{(i,\cdot),\ell}^{(t+1)} - \mathbf{z}_{(i,\cdot),\ell}^{(t)} \in \mathbb{R}^{d{\ell}}.
\end{equation}
For each node, the direction vectors of edge-wise trajectories are contingent upon the messages propagated along the respective edges. However, direct observation of the message on a specific edge is unfeasible, as GNNs only yield the features amalgamated from all neighbors. To address this limitation, we delineate the methodology for computing the direction vector of edge-wise trajectories, applicable to prevalent GNN architectures including GCN, GAT, and GraphSAGE.

\subsection{Trajectories in mainstream GNNs}

\subsubsection{Trajectories in GCN}
Fine-grained trajectories on GCN are nonintuitive since GCN is a transductive model and driving message passing based on the global Laplacian. Here we give the calculation method of edge-wise and node-wise trajectories in GCN based on the deconstruction of the reformulation of GCN. The direction vector of edge FTs in GCN is
\begin{equation}\label{eq_DefEdgeT}\small
\mathcal{T}_{i,(i,j),\ell}^{(t)\to (t+1)}  =  \left(\mathcal{R}\left(\mathbf{L},(i,j)\right) \mathbf{Z}_{\ell-1}^{(t)}  \left(\mathbf{W}_\ell^{(t+1)} - \mathbf{W}_\ell^{(t)}\right) \right)_j,
\end{equation}
where $\mathcal{R}(\mathbf{L},(i,j))$ gives the reduced Laplacian $\mathbf{L}_{i,j}^1$ calculated from $\mathbf{A}_{i,j}^1$ which can revert to a 1-edge graph $\mathcal{G}_{i,j}^{1}=\{ \mathbf{Z}, \mathcal{E}_{i,j}^1 \}$, where $\mathcal{E}_{i,j}^1 = \{ (i,j) \}$ which means only 1 edge is in $\mathcal{E}_{i,j}^1 $. An illustrational example of function of $\mathcal{R}(\cdot,\cdot)$ is as Figure~\ref{fig_FunctionR}.

\begin{figure}[htb]
\centering
\includegraphics[width=0.45\textwidth]{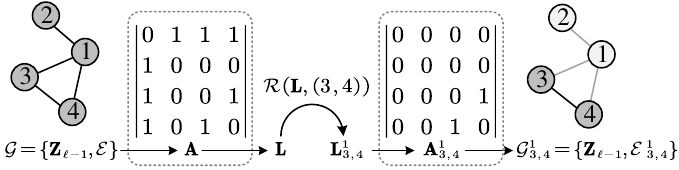}
\caption{An illustration of reduced Laplacian function $\mathcal{R}(\cdot,\cdot)$.}
\label{fig_FunctionR}
\end{figure}

Therefore, given a target node $i$ and maximum observing epoch $E$, edge FT on edge $(i,j)$ can thus be regard as a real-time updated matrix $\mathcal{T}_{(i,j),\ell} \in \mathbb{R}^{E \times d_{\ell}}$ whose row records the history positions.

Following the definition of edge-wise trajectories, node-wise trajectories can be obtained by

\begin{proposition}\label{pro_NTnET}
Given a node $i$, The direction vector of edge-wise and node-wise trajectories satisfies
\begin{equation}\label{eq_NTnET}\small
\mathcal{T}_{i,\ell}^{(t)\to (t+1)} = \sum_{j \in \mathcal{N}(i)} \mathcal{T}_{i,(i,j),\ell}^{(t)\to (t+1)}.
\end{equation}
\end{proposition}
Proof in Appendix~\ref{appen_3}.

\subsubsection{Trajectories in GAT}
Trajectories on GAT is intuitive since GAT is an inductive model which modeled the local MP. Therefore, the direction vector for the edge FTs of node $j$ to $i$ in GAT is the sum of aggregated message from $j$ to $i$ in all attention heads, i.e., in GAT,
\begin{equation}\label{eq_EdgeTrajGAT}\small
\mathcal{T}_{i,(i,j),\ell}^{(t)\to (t+1)} = \sum_{k=1}^{K} a_{i,j}\mathbf{W}_{ \texttt{GAT} }^{(t,k)}\mathbf{z}_{i,\ell-1}^{(t)},
\end{equation}
where $a_{i,j}$ is the attention coefficients between node $i$ and $j$, $\mathbf{W}_{ \texttt{GAT} }^{(t,k)}$ is the trainable matrix of $k^{\text{th}}$ head in epoch $t$.

Furthermore, the direction vector of node-wise trajectory in GAT can be simply calculated by
\begin{equation}\label{eq_NodeTrajGAT}\small
\mathcal{T}_{i,\ell}^{(e)\to (e+1)} = \mathbf{z}^{(t+1)}_{i,\ell} - \mathbf{z}^{(t)}_{i,\ell}.
\end{equation}

\subsubsection{Trajectories in GraphSAGE}
Although GraphSAGE is an inductive model, trajectories in GraphSAGE are nonintuitive since it applies non-linear layer to the aggregated features of all neighbors (including the self node feature). Therefore, edge-wise trajectory of node $j$ to $i$ in GraphSAGE can be obtained by blocking messages of node $h$ satisfies $ \forall h \in \mathcal{N}(i)$ and $h \not= j$. Therefore, we replace the aggregator $\mathrm{A}(\cdot)$ as the blocked aggregator $\mathrm{A}_{\text{blocked}}(\cdot)$ in GraphSAGE and recalculate forward propagation, to thus obtain the edge-wise trajectories, i.e., by denoting $\texttt{SAGE}^{(t)}(\cdot)$ as the $t^{\text{th}}$ epoch forward propagation of GraphSAGE,
\begin{equation}\label{eq_EdgeTrajSAGE}\small
\mathcal{T}_{i,(i,j),\ell}^{(t)\to (t+1)} = \texttt{SAGE}^{(t)}(\mathbf{z}_{i,\ell-1}^{(t)};\mathrm{A}_{\text{blocked}}).
\end{equation}
$\mathrm{A}_{\text{blocked}}(\cdot)$ only aggregate the embedding of node $j$ while banning messages from other neighbors of node $i$. For instance, the blocked pooling aggregator is $\mathrm{A}_{\text{blocked}}(\mathbf{z}_{j,\ell};i) = \mathbf{W}_{\text{pool}} \mathbf{z}_{j,\ell} + b$, where $ \mathbf{W}_{\text{pool}}$ and $\mathbf{b}$ is the trainable matrix and trainable vector in the aggregator. For another instance, the blocked mean aggregator can simply represented as $\mathrm{A}_{\text{blocked}}(\mathbf{z}_{j,\ell};i) = \mathbf{z}_{j,\ell}/|\mathcal{N}(i)|$.
The calculation method for the direction vector in GraphSAGE is the same as Equation~\eqref{eq_NodeTrajGAT}.


\subsection{Trajectory Normalization}\label{sec_norm}

Due to the inconsistency in the initial and convergence positions of the nodes, and given our need to determine their abnormality based on trajectory characteristics, we align the start and end points of all trajectories each time a calculation is performed. This approach focuses solely on extracting pertinent information, such as the magnitude of fluctuations. Consequently, it enables the maximization of the extraction and utilization of the latent information revealed by the trajectories.
We refer to this process as the Normalization of FTs, which inherently introduces a new advantage: Normalized FTs support cross-system transfer. For the same GNN, the primary difference between different training systems is rectified in the varying initial positions of trajectories provided by the datasets. Given the relatively consistent training capabilities of GNNs, Normalization to some extent mitigates the discrepancies caused by the datasets. It aligns different systems with the intrinsic characteristics of the GNN itself, ensuring that the detectors (subsets of FTs, which will be elaborated upon later) remain effective even after cross-system transfer.

\begin{figure}[htb]
\centering
\includegraphics[width=0.40\textwidth]{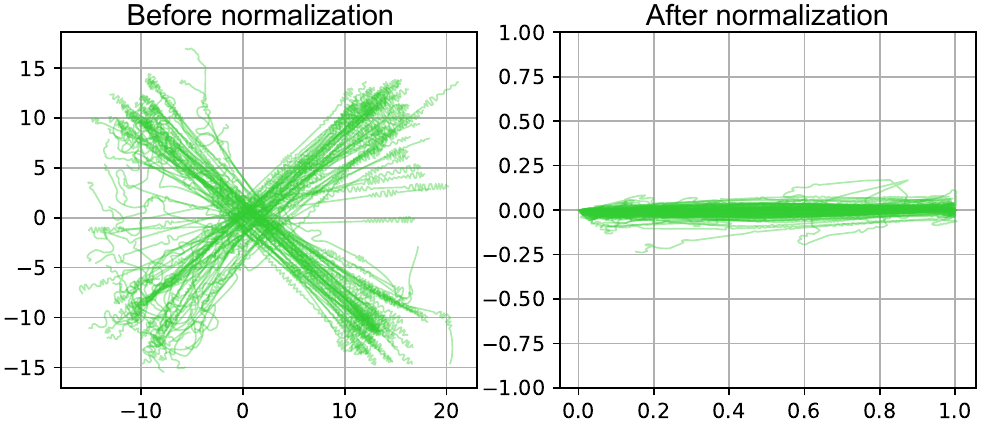}
\caption{The illustrative example for normalization. We report 100 node-wise trajectories which are produced by GCN on the clean Cora dateset and correctly classified, i.e., all trajectories are reliable.} \label{fig_DeloVis}
\end{figure}

Given a set of trajectories $\mathscr{T}=\{\mathscr{T}_1,\ldots,\mathscr{T}_w\}$ where $\mathscr{T}_i \in \mathbb{R}^{d \times l}$ (here we consider broader trajectories, not distinguishing nodes' or edges'), we first parallel transport the starting point of all trajectories to the origin of $d$-dimension axis, to thus translate every trajectory $\mathscr{T}_i,i\in[1,w]$ to $\mathscr{T}^{(0)}_i$. Then, we rotate all trajectories to make the line between the start and the end point of every trajectory parallel to the axis of the first dimension (i.e., for $d=3$ case, parallel to director vector $\imath_{3} = [1,0,0]$). Specifically, for any trajectory $\mathscr{T}^{(0)}_i$, we first calculate the main direction $\zeta_i = \mathscr{T}_i^{(0)}(l)$ where $\mathscr{T}_i^{(0)}(l) \in \mathbb{R}^{d}$ is the $l^{\text{th}}$ position of $\mathscr{T}_i^{(0)}(l)$, and calculate the included angle $\tau_{i,0}$ between $\zeta_i$ and $\imath_{d}$:
$
  \tau_{i,0} = \arccos( \frac{\zeta_i \cdot \imath_{d}}{||\zeta_i||_2}).
$
Then, define the rotate matrix
\begin{equation}\label{eq_RotateM}\small
R_i = \cos \tau_{i,0} I_{d} + T_{\text{L}}(\sin \tau_{i,0},d) + T_{\text{R}}(-\sin \tau_{i,0},d),
\end{equation}
where $I_{d}$ is the identity matrix of size $d \times d$, $T_{\text{L}}(\cdot)$ and $T_{\text{R}}(\cdot)$ represent expanding a scalar to the left and right triangular matrix, respectively. Finally, rotate $z^{(0)}_i$ according to $R_i$, i.e.,
$
\mathscr{T}_i' = R_i \mathscr{T}^{(0)}_i.
$

Subsequently, by applying a min-max scale to $\mathscr{T}_i'$, the normalization process for the given set of trajectories is completed.  We denote $\mathrm{CON}(\cdot,l)$ as the repetitive concatenation of the given vector with repeating $l$ times.
To further clarify the normalization operation, we adopt the same approach as in Section~\ref{sec_motiv}, attaching a unidirectional output, frozen decoder that does not interfere with the training process to obtain 2-dimensional trajectories (see Figure~\ref{fig_motive} for details). The results of the normalization are shown in Figure~\ref{fig_DeloVis}.

\section{Proof of Theorem~\ref{thm_dis}}\label{appen_1}

To prove Theorem~\ref{thm_dis}, we first prove 7 lemmas.

Lemma~\ref{lemma1} establishes the continuity of the GNN function, ensuring that the continuous label space, segmented by the argmax function, can be backpropagated to every layer within the GNN. This guarantees the existence of a continuous latent label space at each layer.

Lemma~\ref{lemma2}  validates the uniqueness of the supervisory signal at each layer, ensuring that once training commences, every layer of the GNN has a fixed and unchanging convergence target.

Based on the conclusions of Lemma~\ref{lemma1} and Lemma~\ref{lemma2}, Lemma~\ref{lemma3} demonstrates the presence of binary opposition within each layer's latent label space. In other words, as node features move towards a certain category within this space, they inherently move away from other regions.

Building on the conclusion of Lemma~\ref{lemma3}, Lemma~\ref{lemma4} proves that node feature trajectories undergo ``secondary convergence'' in the latent label space when subjected to attacks.

Lemma~\ref{lemma5} establishes that the Mean Squared Error (MSE) between node feature trajectories generated by an unattacked GNN follows a normal distribution.

Extending from the findings of Lemma~\ref{lemma4} and Lemma~\ref{lemma5}, Lemma~\ref{lemma6} demonstrates that the MSE between attacked and unattacked node feature trajectories also follows a normal distribution. Furthermore, it establishes that there is a significant and lower-bounded Kullback-Leibler (KL) divergence between this distribution and the distribution described in Lemma~\ref{lemma5}.

Finally, Lemma~\ref{lemma7}, drawing from the conclusion of Lemma~\ref{lemma6} and existing research, analyzes the relationship between KL divergence and distinguishability.

\begin{lemma}\label{lemma1}
Let $\mathcal{M}_\ell: \mathbb{R}^{d_\ell} \to \mathbb{R}^{d_{\ell+1}}$ be the output of the $\ell^{\text{th}}$ layer of a GNN $\mathcal{M}$. Suppose for a subset $V \subset \mathbb{R}^{d_{\ell+1}}$ in the output space, there exists a pre-image set $U \subset \mathbb{R}^{d_\ell}$ in the input space such that $\mathcal{M}_\ell(U) = V$. If $V$ is a continuous value region in $\mathbb{R}^m$, then $U$ is also a continuous region in $\mathbb{R}^n$. In other words, the continuity of the output value region implies the continuity of its corresponding input space region.
\end{lemma}

\begin{proof}

A $L$-layer GNN with function $\mathcal{M}$ that is realizable can be conceptualized through the framework of a set [1], denoted as $\Phi=\left(\left(\mathbf{W}_{\ell}\right)\right)_{\ell=1}^L$, where $\mathbf{W}_{\ell} \in \mathbb{R}^{d_{\ell} \times d_{\ell+1}}$is the trainable matrices. The architecture of $\Phi$, denoted as $S:=\left(d_0, d_1, \ldots, d_L\right)$, represents the number of neurons in each layer of the network. The total number of neurons in the architecture, $D(S):=\sum_{\ell=0}^L d_{\ell}$, and the number of layers, $L=L(S)$, are key characteristics of the network.

The realization of a GNN $\Phi$ is defined through the incorporation of an activation function $\varrho: \mathbb{R} \rightarrow \mathbb{R}$ and a domain of definition $\mathbf{Z} \subset \mathbb{R}^{d_0}$. The realization map of the network, $\mathrm{R}_{\varrho}^{\mathbf{Z}}(\Phi): \mathbf{Z} \rightarrow \mathbb{R}^{d_L}$, maps an input $x$ to the output $x_L$, which is computed as follows:
\begin{align*}
x_0 & :=\mathbf{Z}, \\
x_{\ell} & :=\varrho\left( \mathbf{L} \mathbf{Z}_{\ell-1} \mathbf{W}_{\ell} \right), \quad \text{for } \ell=1, \ldots, L-1, \\
x_L & := \mathbf{L} x_{L-1} \mathbf{W}_{\ell} ,
\end{align*}
where \(\varrho\) acts componentwise. Herein, we posit a more generalized scenario wherein the GNN maps attributes to the realm of real numbers, i.e., $\mathrm{R}_{\varrho}^{\mathbf{Z}}(\Phi): \mathbf{Z} \rightarrow \mathbb{R}$.

We focus on the topological properties of sets of realizations of neural networks with a fixed size. The size of a network can be specified in various ways, but here we consider the set of realizations of networks with a given architecture \(S\) and activation function \(\varrho\), denoted as \( \left\{\mathrm{R}_{\varrho}^{\mathbf{Z}}(\Phi): \Phi \in \mho(S)\right\}\).
Note that the set \(\mho(S)\) of all neural networks with a fixed architecture forms a finite-dimensional vector space. This space is equipped with a norm defined as
\begin{multline}\label{eq_proof_prim_norm}
\|\Phi\|_{\mho(S)}:=\|\Phi\|_{\text{scaling}} \quad \text{for} \quad \Phi=\left(\left(\mathbf{W}_{\ell}\right)\right)_{\ell=1}^L \in \mho(S),
\end{multline}

where \(\|\Phi\|_{\text{scaling}}:=\max_{\ell=1, \ldots, L}\left\|\mathbf{W}_{\ell}\right\|_{\max}\). If the specific architecture of \(\Phi\) is not of concern, we simply write \(\|\Phi\|_{\text{total}}:=\|\Phi\|_{\mho(S)}\). Additionally, if \(\varrho\) is continuous, the realization map is denoted by
\begin{equation}\label{eq_proof_conti_realization}
\mathrm{R}_{\varrho}^{\mathbf{Z}}: \mho(S) \rightarrow \mathcal{M}\left(\mathbf{Z} \right), \Phi \mapsto \mathrm{R}_{\varrho}^{\mathbf{Z}}(\Phi).
\end{equation}
Here we denote $\mathcal{M}=\mathcal{M}(\varrho, S)>0$ as a constant, which represent the output of GNN.

We aim to demonstrate that for two sequences of continuous functions, $(f_n){n\in\mathbb{N}} : \mathbb{R}^{d_0} \rightarrow \mathbb{R}^{d_L}$ and $(g_n){n\in\mathbb{N}} : \mathbb{R}^{d_L} \rightarrow \mathbb{R}^{d_0}$, if $f_n \rightarrow f$ and $g_n \rightarrow g$ under locally uniform convergence, then the composition $g_n \circ f_n$ also converges to $g \circ f$ locally uniformly. To establish this, consider any $R, \epsilon > 0$. Within the ball $B_R(0) \subset \mathbb{R}^{d_0}$, the sequence $(f_n)$ converges uniformly to $f$. This implies that there exists a bound $\mathcal{M} := \sup_{n\in\mathbb{N}} \sup_{|x|\leq R} |f_n(x)| < \infty$, which is feasible due to the continuity of $f$ and all $f_n$, ensuring their boundedness on $B_R(0)$.

Further, within the ball $B_\mathcal{M}(0) \subset \mathbb{R}^{d_L}$, the sequence $(g_n)$ converges uniformly to $g$. Consequently, there exists an $n_1 \in \mathbb{N}$ such that for all $n \geq n_1$ and for all $y \in \mathbb{R}^{d_L}$ with $|y| \leq \mathcal{M}$, the inequality $|g_n(y) - g(y)| < \epsilon$ holds. Additionally, $g$ is uniformly continuous on $B_\mathcal{M}(0)$, allowing us to find a $\delta > 0$ such that $|g(y) - g(z)| < \epsilon$ for all $y, z \in B_\mathcal{M}(0)$ with $|y - z| \leq \delta$. Lastly, the uniform convergence of $f_n$ to $f$ on $B_R(0)$ implies the existence of an $n_2 \in \mathbb{N}$, such that for all $n \geq n_2$ and for all $x \in \mathbb{R}^{d_0}$ with $|x| \leq R$, the condition $|f_n(x) - f(x)| \leq \delta$ is satisfied.

Overall, these considerations show for $n \geq \max\{n_1, n_2\}$ and $x \in \mathbb{R}^d$ with $|x| \leq R$ that
\begin{multline}\label{eq_proof_conti_step1}
|g_n(f_n(x)) - g(f(x))| \\ \leq  |g_n(f_n(x)) - g(f_n(x))| + |g(f_n(x)) - g(f(x))| \leq \epsilon + \epsilon.
\end{multline}

Then, We show that $R^\varrho$ is continuous. Assume that a sequence $(\Phi_n)_{n\in\mathbb{N}} \subset \mho ((d, d_1, ..., d_L))$ given by $\Phi_n = \mathbf{W}_1^{(n)}, \mathbf{W}_L^{(n)}$ satisfies $\Phi_n \rightarrow \Phi = (A_1, b_1), ..., (A_L, b_L) \in \mho ((d, d_1, ..., d_L))$. For $\ell \in \{1, ..., L-1\}$ set
\begin{align}
  \alpha_\ell^{(n)} &: \mathbb{R}^{d_{\ell-1}} \rightarrow \mathbb{R}^{d_\ell}, x \mapsto \varrho_\ell(\mathbf{L}x\mathbf{W}_\ell^{(n)}), \notag \\
  \alpha_\ell &: \mathbb{R}^{d_{\ell-1}} \rightarrow \mathbb{R}^{d_\ell}, x \mapsto \varrho_\ell(\mathbf{L}x\mathbf{W}_\ell), \notag
\end{align}
where $\varrho_\ell := \varrho \times \dots \times \varrho$ denotes the $d_\ell$-fold cartesian product of $\varrho$. Likewise, set
\begin{align}
  \alpha_L^{(n)} &: \mathbb{R}^{d_{L-1}} \rightarrow \mathbb{R}^{d_L}, x \mapsto \mathbf{L}x\mathbf{W}_\ell^{(n)}  \notag \\
  \text{and} \quad \alpha_L &: \mathbb{R}^{d_{L-1}} \rightarrow \mathbb{R}^{d_L}, x \mapsto \mathbf{L}x\mathbf{W}_\ell \notag
\end{align}

By what was shown in Equation~\eqref{eq_proof_conti_step1}, it is not hard to see for every $\ell \in \{1, ..., L\}$ that $\alpha_\ell^{(n)}$ converges uniformly as $n \rightarrow \infty$. By another (inductive) application of Step 1, this shows
\begin{equation}\label{eq_proof_conti_step2}
  R^\varrho(\Phi_n) = \alpha_L^{(n)} \circ \dots \circ \alpha_1^{(n)} \rightarrow \alpha_L \circ \dots \circ \alpha_1 = R^\varrho(\Phi),
\end{equation}
locally uniformly.

Next, Let $\varrho_\ell := \varrho \times \dots \times \varrho$ be the $d_\ell$-fold cartesian product of $\varrho$ in case of $\ell \in \{1, ..., L-1\}$, and set $\varrho_L := \text{id}_{\mathbb{R}^{d_L}}$. For arbitrary $x \in \mathbf{Z}$ and $\Phi = (A_1, b_1), ..., (A_L, b_L) \in \mho(S)$, define inductively $\alpha_x^{(0)}(\Phi) := x \in \mathbb{R}^d = \mathbb{R}^{d_0}$, and for $\ell \in \{0, ..., L-1\}$,
\begin{equation}\label{eq_proof_conti_step3.1}
 \alpha_x^{(\ell+1)}(\Phi) := \varrho_{\ell+1} \left( \mathbf{L}  \alpha_x^{(\ell)}(\Phi) \mathbf{W}_{\ell+1} \right)  \in \mathbb{R}^{d_{\ell+1}}.
\end{equation}
Let $R > 0$ be fixed, but arbitrary. We will prove by induction on $\ell \in \{0, ..., L\}$ that
\begin{align}\label{eq_proof_conti_step3.2}
& \|\alpha_x^{(\ell)}(\Phi)\|_{\infty} \leq \mathcal{M}_{\ell,R} \\
\text{and} \quad & \|\alpha_x^{(\ell)}(\Phi) - \alpha_x^{(\ell)}(\Psi)\|_{\infty} \leq \mathcal{Q}_{\ell,R} \cdot \|\Phi - \Psi\|_{\text{total}} \notag,
\end{align}
for suitable $\mathcal{M}_{\ell,R}, \mathcal{Q}_{\ell,R} > 0$ and arbitrary $x \in \mathbf{Z}$ and $\Phi, \Psi \in \mho(S)$ with $\|\Phi\|_{\text{total}}, \|\Psi\|_{\text{total}} \leq R$. This will imply that $R^\varrho_{\mathbf{Z}}$ is locally Lipschitz, since clearly $R^\varrho_{\mathbf{Z}}(\Phi)(x) = \alpha_x(\Phi)$, and hence
\begin{multline}\label{eq_proof_conti_step3.3}
\|R^\varrho_{\mathbf{Z}}(\Phi) - R^\varrho_{\mathbf{Z}}(\Psi)\|_{\text{sup}} \\ = \sup_{x \in \mathbf{Z}} |\alpha_x(\Phi) - \alpha_x(\Psi)| \leq \mathcal{Q}_{L,R} \cdot \|\Phi - \Psi\|_{\text{total}}.
\end{multline}

Next, Let $\varrho_{\ell}:=\varrho \times \cdots \times \varrho$ be the $d_{\ell}$-fold cartesian product of $\varrho$ in case of $\ell \in\{1, \ldots, L-1\}$, and set $\varrho_L:=\operatorname{id}_{\mathbb{R}^{d_L}}$. For arbitrary $x \in \mathbf{Z}$ and $\Phi=\left(\mathbf{W}_1, \ldots,\mathbf{W}_L\right) \in \mho (S)$, define inductively $\alpha_x^{(0)}(\Phi):=x \in \mathbb{R}^d=\mathbb{R}^{d_0}$,  for $\ell \in\{0, \ldots, L-1\}$
\begin{align}\label{eq_proof_conti_step3.1}
\alpha_x^{(\ell+1)}(\Phi):=\varrho_{\ell+1}\left(\mathbf{L}  \alpha_x^{(\ell)}(\Phi) \mathbf{W}_{\ell+1}\right) \in \mathbb{R}^{d_{\ell+1}} \notag \\
 \quad \text { for } \quad \ell \in\{0, \ldots, L-1\}.
\end{align}
Let $R>0$ be fixed, but arbitrary. We will prove by induction on $\ell \in\{0, \ldots, L\}$ that
\begin{align}\label{eq_proof_conti_step3.2}
& \left\|\alpha_x^{(\ell)}(\Phi)\right\|_{\ell^{\infty}} \leq \mathcal{M}_{\ell, R} \notag \\
 \text {and} \quad & \left\|\alpha_x^{(\ell)}(\Phi)-\alpha_x^{(\ell)}(\Psi)\right\|_{\ell^{\infty}} \leq \mathcal{Q}_{\ell, R} \cdot\|\Phi-\Psi\|_{\text {total }},
\end{align}
for suitable $\mathcal{M}_{\ell, R}, \mathcal{Q}_{\ell, R}>0$ and arbitrary $x \in \mathbf{Z}$ and $\Phi, \Psi \in \mho (S)$ with $\|\Phi\|_{\text {total }},\|\Psi\|_{\text {total }} \leq R$.
This will imply that $\mathrm{R}_{\varrho}^{\mathbf{Z}}$ is locally Lipschitz, since clearly $\mathrm{R}_{\varrho}^{\mathbf{Z}}(\Phi)(x)=\alpha_x^{(L)}(\Phi)$, and hence
\begin{multline}\label{eq_proof_conti_step3.3}
\left\|\mathrm{R}_{\varrho}^{\mathbf{Z}}(\Phi)-\mathrm{R}_{\varrho}^{\mathbf{Z}}(\Psi)\right\|_{\mathrm{sup}}= \\
\sup _{x \in \mathbf{Z}}\left|\alpha_x^{(L)}(\Phi)-\alpha_x^{(L)}(\Psi)\right| \leq \mathcal{Q}_{L, R} \cdot\|\Phi-\Psi\|_{\text {total }} .
\end{multline}

The case $\ell=0$ is trivial: On the one hand, $\left|\alpha_x^{(0)}(\Phi)-\alpha_x^{(0)}(\Psi)\right|=0 \leq\|\Phi-\Psi\|_{\text {total }}$. On the other hand, since $\mathbf{Z}$ is bounded, we have $\left|\alpha_x^{(0)}(\Phi)\right|=|x| \leq \mathcal{M}_0$ for a suitable constant $\mathcal{M}_0=\mathcal{M}_0(\mathbf{Z})$.
For the induction step, let us write $\Psi=\left( \mathbf{V}_1 , \ldots, \mathbf{V}_L \right)$, and note that
\begin{multline}\label{eq_proof_conti_step3.4}
\left\|\mathbf{L}  \alpha_x^{(\ell)}(\Phi) \mathbf{W}_{\ell+1}  \right\|_{\ell \infty}  \\
\leq d_{\ell}  \left\| \mathbf{L}\right\|_{\max } \cdot  \left\| \mathbf{W}_{\ell+1}\right\|_{\max } \cdot  \left\|\alpha_x^{(\ell)}(\Phi)\right\|_{\ell^{\infty}}  \\
 \leq\left(1+d_{\ell} \mathcal{M}_{\ell, R}\right) \cdot\|\Phi\|_{\text {total }}=: K_{\ell+1, R} .
\end{multline}
Clearly, the same estimate holds with $A_{\ell+1}, b_{\ell+1}$ and $\Phi$ replaced by $B_{\ell+1}, c_{\ell+1}$ and $\Psi$, respectively. Next, observe that with $\varrho$ also $\varrho_{\ell+1}$ is locally Lipschitz. Thus, there is $\Gamma_{\ell+1, R}>0$ with
\begin{multline}\label{eq_proof_conti_step3.5}
\left\|\varrho_{\ell+1}(x)-\varrho_{\ell+1}(y)\right\|_{\ell^{\infty}} \leq \Gamma_{\ell+1, R} \cdot\|x-y\|_{\ell^{\infty}} \\
\text { for all } x, y \in \mathbb{R}^{d_{\ell+1}} \text { with }\|x\|_{\ell^{\infty}},\|y\|_{\ell^{\infty}} \leq K_{\ell+1, R} .
\end{multline}

On the one hand, this implies
\begin{align} \small
& \left\|\alpha_x^{(\ell+1)}(\Phi)\right\|_{\ell^{\infty}} \notag \\
& \ \leq\left\|  \varrho_{\ell+1}\left( \mathbf{L} \alpha_x^{(\ell)}(\Phi) \mathbf{W}_{\ell+1}  \right)-\varrho_{\ell+1}(0)\right\|_{\ell^{\infty}}+\left\|\varrho_{\ell+1}(0)\right\|_{\ell^{\infty}} \notag \\
& \ \leq \Gamma_{\ell+1, R}\left\| \mathbf{L} \alpha_x^{(\ell)}(\Phi) \mathbf{W}_{\ell+1}  \right\|_{\ell^{\infty}}+\left\|\varrho_{\ell+1}(0)\right\|_{\ell^{\infty}} \notag \\
& \ \leq \Gamma_{\ell+1, R} K_{\ell+1, R}+\left\|\varrho_{\ell+1}(0)\right\|_{\ell^{\infty}}=: \mathcal{M}_{\ell+1, R}.
\end{align}

On the other hand, we also get
\begin{align} \small
& \left\|\alpha_x^{(\ell+1)}(\Phi)-\alpha_x^{(\ell+1)}(\Psi)\right\|_{\ell^{\infty}} \notag \\
& =\left\|\varrho_{\ell+1}\left( \mathbf{L} \alpha_x^{(\ell)}(\Phi) \mathbf{W}_{\ell+1}  \right)-\varrho_{\ell+1}\left( \mathbf{L} \alpha_x^{(\ell)}(\Phi) \mathbf{V}_{\ell+1}  \right)\right\|_{\ell^{\infty}} \notag \\
& \leq \Gamma_{\ell+1, R} \cdot\left\|\left( \mathbf{L} \alpha_x^{(\ell)}(\Phi) \mathbf{W}_{\ell+1}  \right)-\left( \mathbf{L} \alpha_x^{(\ell)}(\Phi) \mathbf{V}_{\ell+1}  \right)\right\|_{\ell^{\infty}} \notag \\
& \leq \Gamma_{\ell+1, R} \cdot\Big( \mathbf{L} \alpha_x^{(\ell)}(\Phi) \left\|\left(\mathbf{W}_{\ell+1} - \mathbf{V}_{\ell+1}\right) \right\|_{\ell \infty} \Big) \notag \\
& \leq \Gamma_{\ell+1, R} \cdot\Big(d_{\ell} \cdot\|\Phi-\Psi\|_{\text {total }} \cdot\left\|\alpha_x^{(\ell)}(\Phi)\right\|_{\ell \infty} \notag \\
& \leq \Gamma_{\ell+1, R} \cdot d_{\ell} \mathcal{M}_{\ell, R} \cdot\|\Phi-\Psi\|_{\text {total }} \notag \\
& =: \mathcal{Q}_{\ell+1, R} \cdot\|\Phi-\Psi\|_{\text {total }} .
\end{align}

Hence, we have the following conclusion: if $\varrho$ is globally Lipschitz continuous, then there is a constant $\mathcal{M}>0$ such that
\begin{equation}\label{eq_priif_conti_conclu1}
\operatorname{Lip}\left(\mathrm{R}_{\varrho}^{\mathbf{Z}}(\Phi)\right) \leq \mathcal{M} \cdot\|\Phi\|_{\text {scaling }}^L \quad \text { for all } \Phi \in \mho (S)
\end{equation}

The aforementioned conclusion substantiates the continuity of $\mathcal{M}$ within the context of Backpropagation training.

Assume $V$ is a continuous value region in the output space $\mathbb{R}^{d_{\ell+1}}$. This implies that for any sequence $\{y_k\}$ in $V$ that converges to a point $y \in V$, there exists a sequence $\{x_k\}$ in $U$ such that $\mathcal{M}_\ell(x_k) = y_k$ for all $k$ and $\mathcal{M}_\ell(x_k) \to \mathcal{M}_\ell(x)$, where $x$ is the pre-image of $y$ under $\mathcal{M}_\ell$.

Since $\mathcal{M}_\ell$ is continuous, the convergence $\mathcal{M}_\ell(x_k) \to \mathcal{M}_\ell(x)$ implies $x_k \to x$. Therefore, the pre-image of every convergent sequence in $V$ is a convergent sequence in $U$, which makes $U$ a continuous region in $\mathbb{R}^{d_{\ell}}$. The continuity of $\mathcal{M}_\ell$ ensures that the limit points in $V$ correspond to limit points in $U$, hence preserving the continuity of the region from the output space to the input space.

\end{proof}

\begin{lemma}\label{lemma2}

Let $\mathcal{L}^{(\ell)}$ denote the label space for layer $\ell$. For any node $i$ and layer $\ell$, let $\mathcal{M}_\ell(\mathbf{Z}; t)$ represent the output of layer $\ell$ at training epoch $t$, and let $\mathbf{\ell}_i^{(\ell)}$ denote the convergence position in $\mathcal{L}^{(\ell)}$ for node $i$. Then, the following relation is postulated:
\begin{equation}\label{eq_thm_fixed}
  \forall t, t' \geq t_0, \ \mathcal{M}_\ell (\mathbf{Z};t) \rightarrow \mathbf{y}^{(\ell)} \ \text{ and } \ \mathcal{M}_\ell (\mathbf{Z};t') \rightarrow \mathbf{y}^{(\ell)},
\end{equation}
where $t_0$ represents the epoch at which training commences, and $\rightarrow$ is employed to denote the convergence of the layer's outputs towards a specific value as the training progresses.
\end{lemma}

\begin{proof}
Consider a GNN $\mathcal{M}$, where each layer $\ell$ applies a transformation to the input features $\mathbf{Z}$, resulting in an output in the label space $\mathcal{L}^{(\ell)}$. The transformation at layer $\ell$ is denoted by $\mathcal{F}_\ell(\mathbf{Z}, \mathbf{W}^{(\ell)})$, where $\mathbf{W}^{(\ell)}$ denotes the weight matrix at that layer.

During the training process, the weights $\mathbf{W}^{(\ell)}$ are iteratively updated. Let $\mathbf{W}^{(\ell)}_t$ represent the weight matrix at training epoch $t$. The update rule for the weights is given by:
\begin{equation}
    \mathbf{W}^{(\ell)}_{t+1} = \mathbf{W}^{(\ell)}_t - \eta \nabla_{\mathbf{W}^{(\ell)}} \mathcal{L},
\end{equation}
where $\eta$ is the learning rate, and $\nabla_{\mathbf{W}^{(\ell)}} \mathcal{L}$ represents the gradient of the loss function $\mathcal{L}$ with respect to $\mathbf{W}^{(\ell)}$.

Convergence is defined as the condition where the updates to $\mathbf{W}^{(\ell)}$ become infinitesimally small, which can be formally expressed as:
\begin{equation}
    \lim_{t \to \infty} \|\mathbf{W}^{(\ell)}_{t+1} - \mathbf{W}^{(\ell)}_t\| = 0.
\end{equation}
This condition implies that for sufficiently large $t$, $\mathbf{W}^{(\ell)}_t$ approaches a quasi-stationary state, denoted as $\mathbf{W}^{(\ell)}_*$, where the changes in the weights are negligible.

As a result of this convergence, the output of the GNN at layer $\ell$ also stabilizes. Formally, we can express this stabilization as:
\begin{equation}
    \lim_{t \to \infty} \mathcal{M}_\ell(\mathbf{Z}; t) = \mathcal{F}_\ell(\mathbf{Z}, \mathbf{W}^{(\ell)}_*).
\end{equation}
The convergence position $\mathbf{\ell}_i^{(\ell)}$ for any node $i$ in layer $\ell$ is then defined as the output of $\mathcal{M}_\ell(\mathbf{Z}; t)$ in this stable state, i.e., $\mathcal{F}_\ell(\mathbf{Z}_i, \mathbf{W}^{(\ell)}_*)$. Since $\mathbf{W}^{(\ell)}_*$ becomes invariant after convergence, the convergence position $\mathbf{Y}_i^{(\ell)}$ also becomes invariant.
\end{proof}

\textsc{Remark}: The theorem articulates that, once training has been initiated for a GNN $\mathcal{M}$, the convergence position $\mathbf{y}_i^{(\ell)}$ for the output of any given layer $\ell$, within the label space $\mathcal{L}^{(\ell)}$, remains invariant across the training epochs. Specifically, for all epochs $t, t' \geq t_0$, the outputs of layer $\ell$, $\mathcal{M}_\ell(\mathbf{Z}; t)$ and $\mathcal{M}_\ell(\mathbf{Z}; t')$, consistently evolve towards the same convergence position $\mathbf{y}_i^{(\ell)}$. The notation $\rightarrow$ is employed to denote the convergence of the layer's outputs towards a specific value as the training progresses. This assertion, thus, rigorously defines the stability and persistence of the target (convergence position) for each layer's output, notwithstanding the progression through successive training epochs.

\begin{lemma}\label{lemma3}
Let \( \mathbb{R}^M \) and \( \mathbb{R}^N \) be real spaces with \( M > N \). Consider a transformation \( T: \mathbb{R}^M \to \mathbb{R}^N \) and regions \( \{R_1, R_2, \ldots, R_N\} \) in \( \mathbb{R}^N \), partitioned via argmax. Given unique extremal points \( \{p_1, p_2, \ldots, p_N\} \) in \( \mathbb{R}^N \) corresponding to each \( R_i \), for any \( x \in \mathbb{R}^M \) converging to \( T^{-1}(p_i) \), it diverges from \( T^{-1}(p_j) \), \( \forall j \neq i \).
\end{lemma}

\begin{proof}
Let \( T: \mathbb{R}^M \to \mathbb{R}^N \) be a transformation where \( M > N \). For each region \( R_i \) in \( \mathbb{R}^N \), partitioned via argmax, associate a unique extremal point \( p_i \). The pre-image of \( p_i \) under \( T \) is denoted as \( q_i \subseteq \mathbb{R}^M \). Introduce a distance metric \( d: \mathbb{R}^M \times \mathbb{R}^M \to \mathbb{R} \) to measure distances in \( \mathbb{R}^M \).

Assume \( T \) is continuous and differentiable. The distinctness of each region \( R_i \) in \( \mathbb{R}^N \) is ensured by the uniqueness of the extremal point \( p_i \) and the argmax partitioning mechanism. Analyze the trajectory of a point \( x \in \mathbb{R}^M \) as it evolves over time. \( x \) is initially at a position \( x_0 \in \mathbb{R}^M \), and the goal is to observe the behavior of \( x \) as it moves closer to \( q_i \), the pre-image of \( p_i \). Define the trajectory of \( x \) as \( \gamma: \mathbb{N} \to \mathbb{R}^M \), with \( \gamma(0) = x_0 \). As \( \gamma \) evolves, if \( \gamma(t+1) \) is closer to \( q_i \) than \( \gamma(t) \), then \( d(\gamma(t+1), q_i) < d(\gamma(t), q_i) \), indicating that \( x \) is moving towards \( q_i \).

The corresponding trajectory in \( \mathbb{R}^N \) is \( T(\gamma(t)) \). Due to \( T \)'s continuity and differentiability, if \( \gamma(t) \) approaches \( q_i \), then \( T(\gamma(t)) \) approaches \( p_i \). Considering \( p_i \)'s uniqueness and the argmax-defined separation of regions, if \( T(\gamma(t)) \) is nearing \( p_i \), it is simultaneously distancing itself from \( p_j \), \( \forall j \neq i \). Thus, \( \gamma(t) \) is not just approaching \( q_i \) but also receding from \( q_j \), \( \forall j \neq i \), where \( q_j \) denotes the pre-image of \( p_j \) under \( T \).

Formally, for \( \gamma(t) \) converging to \( q_i \), \( \lim_{t \to \infty} MSE(\gamma(t), q_i) = 0 \). Simultaneously, for \( j \neq i \), \( \lim_{t \to \infty} MSE(\gamma(t), q_j) \) increases, reflecting divergence. In conclusion, the trajectory of \( x \) in \( \mathbb{R}^M \), as it approaches the pre-image \( q_i \) under a continuous and differentiable transformation \( T \), necessitates the simultaneous divergence from the pre-images of other extremal points \( q_j \), \( \forall j \neq i \). This behavior is underpinned by the properties of \( T \), the unique extremal points in \( \mathbb{R}^N \), and the distinct partitioning of the space. The proof thus establishes a rigorous mathematical relationship between the convergence to a particular pre-image in \( \mathbb{R}^M \) and the necessary divergence from the others, as articulated in the lemma.
\end{proof}

\begin{lemma}\label{lemma4}
Let \( \mathcal{M}: \mathcal{X} \rightarrow \mathcal{Y} \) be a continuous mapping by a GNN, with \( \mathcal{Y} \) partitioned into disjoint regions \( \{R_1, R_2, ..., R_c\} \). Under normal conditions, the trajectory \( \tau: \mathbb{N} \rightarrow \mathcal{X} \) of \( x_v \) converges to \( \mathcal{M}^{-1}(R_t) \). Under an adversarial attack targeting misclassification into \( R_a \), there exists \( T, M \in \mathbb{N} \) such that \( \tau \) initially converges towards \( \mathcal{M}^{-1}(R_t) \) for \( t < T \), then deviates for \( t \in [T, T+M) \), and finally converges to \( \mathcal{M}^{-1}(R_a) \) for \( t \geq T+M \).
\end{lemma}

\begin{proof}

Consider a node \( v \in V \) with its feature representation \( x_v \in \mathcal{X} \), and let \( \mathcal{M}: \mathcal{X} \rightarrow \mathcal{Y} \) be a GNN's mapping from the feature space \( \mathcal{X} \) to the label space \( \mathcal{Y} \). The label space \( \mathcal{Y} \) is partitioned into \( c \) disjoint regions \( \{R_1, R_2, ..., R_c\} \), such that \( R_i \cap R_j = \emptyset \) for \( i \neq j \), and each region corresponds to a unique class label.

Define the trajectory of \( x_v \) over time as a function \( \tau: \mathbb{N} \rightarrow \mathcal{X} \), where \( \mathbb{N} \) denotes the set of natural numbers representing discrete time steps (epochs) in the training process.

In the absence of adversarial influence, the trajectory \( \tau \) is governed by the iterative process \( \tau(t+1) = \mathcal{M}(\tau(t); \mathcal{N}(v)), \forall t \in \mathbb{N} \), where \( f \) represents the update function defined by the GNN, and \( \mathcal{N}(v) \) denotes the neighborhood of node \( v \). It's assumed that \( f \) is designed such that \( \lim_{t \rightarrow \infty} \tau(t) = \mathcal{M}^{-1}(R_t) \), ensuring that the node's feature representation converges to the pre-image of its true label region \( R_t \) in \( \mathcal{Y} \).

Under adversarial attack, the trajectory \( \tau \) is influenced by additional perturbations aimed at misclassifying \( v \). The attack can be modeled as an additive perturbation to the feature representation over time. Therefore, the evolution of \( \tau \) can be described as follows:

\begin{itemize}
  \item In the Initial Phase (\( t < T \)), $$ \tau(t+1) = \mathcal{M}(\tau(t); \mathcal{N}(v)) + \epsilon_t $$, where \( \epsilon_t \) represents the initial negligible effect of the attack, and the trajectory is still predominantly guided towards \( \mathcal{M}^{-1}(R_t) \).
  \item In the Transition Phase (\( t \in [T, T+M) \)), $$ \tau(t+1) = \mathcal{M}(\tau(t); \mathcal{N}(v)) + \delta_t $$, where \( \delta_t \) represents the growing influence of the attack, and \( \tau \) starts deviating from its original path, indicating the onset of the adversarial effect.
  \item In the Final Phase (\( t \geq T+M \)), $$ \tau(t+1) = \mathcal{M}(\tau(t), \mathcal{N}(v)) + \zeta_t $$, with \( \zeta_t \) signifying the full convergence towards the adversarial target. The trajectory is now aligned with the pre-image of \( R_a \), i.e., \( \lim_{t \rightarrow \infty} \tau(t) = \mathcal{M}^{-1}(R_a) \).
\end{itemize}

In conclusion, the trajectory \( \tau \) of a node's feature representation in a GNN shows distinct convergence behaviors under normal and adversarial conditions, as mathematically modeled by the iterative processes above. This substantiates the lemma's claim by providing a more detailed mathematical analysis of the trajectory's evolution.

\end{proof}

\begin{myDef}
The learning capability of GNN, denoted as $C$, is a real-valued metric that quantifies the upper bound of the cosine of the angle between consecutive vectors in the trajectory of the network's nodes' hidden features. Specifically, it holds that $\vec{v}_t \cdot \vec{v}_{t+1} \leq C$. This measure reflects the consistency and directness of the learning path in the high-dimensional feature space, with a higher $C$ indicating a more direct or consistent path towards the target representation.
\end{myDef}

\begin{lemma}\label{lemma5}
In non-adversarial scenarios, the feature trajectories formed by a GNN with learning capability $C$ satisfy:
\begin{multline}\label{eq_non_adv_distri}
  \{ MSE(\mathcal{T}_i, \mathcal{T}_j) : \forall i, j \in \mathcal{V}, i \neq j \} \sim P_{non}  \\
  := \mathrm{Norm}\left(\eta^2 (1 - C), \frac{\eta^4 (1 - C^2)}{12n}\right), \\
  \text{ s.t.: } \mathcal{T}_{\cdot} \mathrm{\ formed\ in\ non}\text{-}\mathrm{adversarial\ scenarios.}
\end{multline}
\end{lemma}

\begin{proof}

Let us denote the squared distance between the \(i^{th}\) steps of trajectories \( \mathcal{T}_1 \) and \( \mathcal{T}_2 \) by \( X_i = (\vec{X}_{1i} - \vec{X}_{2i})^2 \), for each \( i \) in the index set \( \{1, 2, ..., n\} \). The Mean Squared Error between \( \mathcal{T}_1 \) and \( \mathcal{T}_2 \), then, is nothing but the arithmetic mean of these individual squared distances, succinctly expressed as \( MSE(\mathcal{T}_1, \mathcal{T}_2) = \frac{1}{n} \sum_{i=1}^{n} X_i \).

Given the nature of the problem, where each step of the trajectory is chosen independently and the distribution of directions is bounded due to the inner product constraint, we can assert that the \( X_i \)'s are independent and identically distributed (i.i.d.). This i.i.d. nature of \( X_i \) and the large number of steps \( n \) guide us to invoke the central limit theorem (CLT) [2], a cornerstone of probability theory. The CLT asserts that the sum of a large number of i.i.d. random variables tends to follow a normal distribution, regardless of the original distribution of the individual variables. Thus, as per the CLT, the distribution of the sum \( \sum_{i=1}^{n} X_i \) approximates a normal distribution, specifically \( N(n\mu_X, n\sigma_X^2) \), where \( \mu_X \) and \( \sigma_X^2 \) represent the mean and variance of a single \( X_i \), respectively.

Delving into the constraints of our trajectories, the bounded nature of the cosine of the angle between consecutive steps translates into a specific structure for the distribution of \( X_i \). The expected value, \( \mu_X = E[X_i] \), is derived from the geometrical constraints of the problem, which simplifies to \( \eta^2 (1 - C) \). Similarly, the variance of \( X_i \), denoted \( \sigma_X^2 = Var[X_i] \), takes into account the bounded variability of the directional changes, and is given by \( \eta^4 \frac{1 - C^2}{12} \). This leads us to a crucial point: the sum of the \( X_i \)'s follows the distribution \( \sum_{i=1}^{n} X_i \sim N(n\eta^2 (1 - C), n\eta^4 \frac{1 - C^2}{12}) \).

However, what we seek is the distribution of the mean of these \( X_i \)'s, namely the MSE. Leveraging the properties of the normal distribution, which allow for the linear transformation of variables, we find that the distribution of the mean \( MSE(\mathcal{T}_1, \mathcal{T}_2) \) is also normal. It follows the distribution \( N(\eta^2 (1 - C), \frac{\eta^4 (1 - C^2)}{12n}) \), as scaling a normal variable by a constant factor adjusts the mean and variance accordingly.

In conclusion, for a sufficiently large number of steps \( n \), the distribution of \( MSE(\mathcal{T}_1, \mathcal{T}_2) \) conforms closely to a normal distribution with the mean \( \mu = \eta^2 (1 - C) \) and variance \( \sigma^2 = \frac{\eta^4 (1 - C^2)}{12n} \). This normal approximation facilitates the practical analysis of MSE in the context of trajectories, offering a robust statistical foundation for understanding the variability and expectation of the squared distances between any two given trajectories.

\end{proof}

\begin{lemma}\label{lemma6}
Let $\mathcal{T}_{adv}$ represent the arbitrary feature trajectories formed under adversarial scenarios, and let $\mathcal{T}_{non}$ denote the arbitrary feature trajectories formed under non-adversarial scenarios. If we define ${MSE(\mathcal{T}_{adv}, \mathcal{T}_{non})}$ to follow the distribution $P_{adv}$, it holds that the Kullback-Leibler divergence $D_{KL}(P_{non}||P_{adv})$ satisfies
\begin{equation}\label{eq_KL}
  D_{KL}(P_{non}||P_{adv}) \geq \frac{6 q^4}{n - C^2 n}.
\end{equation}
\end{lemma}

\begin{proof}
For $\mathcal{T}_{adv}$, consider an initial offset leading to a squared distance $Y \sim \mathrm{Norm}(\mu_Y, \sigma_Y^2)$ for the first $q$ steps. The MSE for the remaining $n-q$ steps is modeled similarly to $\mathcal{T}_1$, resulting in a total MSE of $MSE_{P_{adv}} = \frac{qY + (n-q)MSE(\mathcal{T}_1, \mathcal{T}_2)}{n}$.

The mean and variance of $MSE_{P_{adv}}$ are $\mu_{P_{adv}} = \frac{q\mu_Y + (n-q)\eta^2 (1 - C)}{n}$ and $\sigma_{P_{adv}}^2 = \frac{q\sigma_Y^2 + (n-q)\frac{\eta^4 (1 - C^2)}{12}}{n^2}$, respectively. Applying the Central Limit Theorem, $MSE_{total}$ approximates a normal distribution for large $n$. Then, we have
\begin{multline}\small
D_{KL}(P_{non}||P_{adv}) \\ = \log \frac{\sigma_Q}{\sigma_{P_{non}}} + \frac{\sigma_{P_{non}}^2 + (\mu_{P_{non}} - \mu_{P_{adv}})^2}{2\sigma_Q^2} - \frac{1}{2}.
\end{multline}

Assuming the conditions such as the magnitude of the initial offset ($y\mu_Y$) and the number of initial offset steps ($q$) are significant, the term $(\mu_{P_{non}} - \mu_{total})^2$ becomes large, leading to a large $D_{KL}(P_{non}||P_{adv})$. Specifically, if $D_{KL}(P_{non}||P_{adv}) > \theta$ for a significant threshold $\theta$, it indicates a substantial divergence between the two distributions, affirming the theorem's assertion of significant separability.

To further simplify, if the initial offset is significant, we can approximate:
\begin{equation}\small
D_{KL}(P_{non}||P_{adv}) \approx \frac{(\mu_{P_{non}} - \mu_{P_{adv}})^2}{2\sigma_{total}^2}.
\end{equation}

And under the assumption that $\sigma_{P_{adv}}^2$ is not significantly larger than $\sigma_{P_{non}}^2$, we can provide an approximate lower bound:
{\small
\begin{multline}\label{eq_KL_1}
  D_{KL}(P_{non}||P_{adv}) \geq \frac{(\mu_{P_{non}} - \mu_{P_{adv}})^2}{2\sigma_{P_{non}}^2} \\
  = \frac{\left(\eta^2 (1 - C) - \frac{q\mu_Y + (n-q)\eta^2 (1 - C)}{n}\right)^2}{2 \cdot \frac{\eta^4 (1 - C^2)}{12n}}.
\end{multline}
}
Since $\mu_Y$ represents the mean MSE for trajectories in $\mathcal{S}_2$. It includes the mean MSE $\mu_1$ from $\mathcal{S}_1$ and an additional term $q\eta^2$ accounting for the initial $q$ diverging steps, thus $\mu_Y = \eta^2(1-C) + q\eta^2$. Considering the gradual impact of adversarial attacks and the adaptive nature of GNNs, it's estimated that approximately 10\% to 30\% of the epochs contribute to the trajectory's deviation towards an incorrect classification. This ratio reflects the balance between the model's initial resistance and the eventual influence of a sustained, undetected attack.

Substantiate these conclusions into Equation~\eqref{eq_KL_1}, yielding
\begin{multline}\label{eq_KL_2}
D_{KL}(P_{non}||P_{adv}) \geq \frac{-6 q^2 (\mu_1 + \eta^2 (-1 + C + q))^2}{(-1 + C) (1 + C) \eta^4 n} \\ \geq \frac{6 q^4}{n - C^2 n}.
\end{multline}

This approximation emphasizes the significance of the mean difference in contributing to the KL divergence and provides a simpler way to evaluate the distinguishability between the two sets of trajectories.
\end{proof}

\begin{figure*}[htb]
\centering
\includegraphics[width=0.99\textwidth]{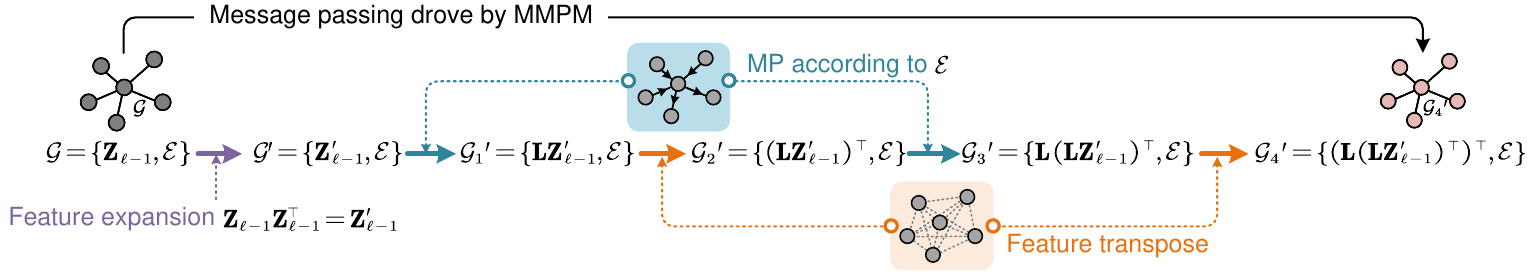}
\caption{The general workflow of MMPM.}
\label{fig_mmpm}
\end{figure*}

\begin{lemma}\label{lemma7}
Given two distinct datasets with probability distributions $P_{non}$ and $P_{adv}$ respectively, if the KL divergence between $P_{non}$ and $P_{adv}$ has a significant lower bound, it implies that the datasets are statistically distinguishable and potentially separable by discriminative models.
\end{lemma}

\begin{proof}

In Equation~\eqref{eq_KL_2}, $C$ represents the dot product of adjacent directional vectors. Ideally, when the angle between them is zero, the value of $C$ is 1. In practice, the smaller the angle, the stronger the learning capability of the GNN. Therefore, the lower bound of the KL divergence in Equation~\eqref{eq_KL_2} is contingent upon the learning capacity of the GNN. Given that Loukas [3] has theoretically demonstrated the infinite learning capacity of GNNs, it is feasible to postulate that $C$ infinitely approximates 1, yet remains slightly less than 1. Based on this premise, we can write
\begin{equation}\label{eq_lemma7}
  \lim_{C \to 1^-} \frac{6 q^4}{n - C^2 n} \to +\infty \Longrightarrow D_{KL}(P_{non}||P_{adv}) \to +\infty.
\end{equation}

The KL divergence is a well-established metric for quantifying the disparity between two probability distributions [4]. When the KL divergence between two distributions P and Q exhibits a significant lower bound, it signifies a pronounced statistical dissimilarity between the datasets. This statistical dissimilarity is crucial as it suggests that each dataset is characterized by distinctive statistical properties, thus laying the groundwork for the potential separability of the datasets.

The pronounced statistical differences highlighted by the KL divergence provide a theoretical basis for the application of discriminative models, such as Support Vector Machines (SVMs) or Neural Networks (NNs). These models exploit the distinct statistical features inherent in each dataset to effectively separate them, as explored in various studies [5]. The substantial KL divergence between distributions thus underpins the models' ability to differentiate between the datasets.

\end{proof}

We now restate Theorem~\ref{thm_dis} and provide a proof.

\section{Proof of Proposition~\ref{pro_upperbound}}\label{appen_2}

In this proof, for the sake of simplifying notation, we use $\mathbf{W}$ without a subscript to represent the weight matrix $\mathbf{W}_{\ell-1}$ of the $\ell-1^{\text{th}}$ layer. We assume that the activation function of GCN is Sigmoid, denoted as $\mathbb{S}(\cdot)$. Based on the forward propagation of GCN,
{\small
    \begin{multline}\label{eq_DefOfDomain}
   (\mathbf{z}_{i,\ell}^{(t+1)} - \mathbf{z}_{i,\ell}^{(t)}) \cdot (\mathbf{z}_{i,\ell}^{(t+2)} - \mathbf{z}_{i,\ell}^{(t-1)}) \\
   \geq \max_{\forall i,j \in \mathcal{V}}( (\mathbf{L}\mathbf{Z}_{\ell-1}\mathbf{W}^{(t+1)} - \mathbf{L}\mathbf{Z}_{\ell-1}\mathbf{W}^{(t)})_{i,\cdot} \\
    \cdot (\mathbf{L}\mathbf{Z}_{\ell-1}\mathbf{W}^{(t+2)} - \mathbf{L}\mathbf{Z}_{\ell-1}\mathbf{W}^{(t+1)})_{j,\cdot} )\\
    \geq \max_k \| (\mathbf{L}\mathbf{Z}_{\ell-1}\mathbf{W}^{(-)})_k \|^2_2,
   \end{multline}\
}
   where $\mathbf{W}^{(-)}$ is the difference weight matrix between adjacent epochs, and it is the key item that can determine the feasible domain of node signal variation, since $\mathbf{L}$ and $\mathbf{Z}_{\ell-1}$ is constant matrix before the $e^{\text{th}}$-epoch message passing. According to the chain rule,
   \begin{equation}\label{eq_chain}\small
     \mathbf{W}^{(-)}
     = \eta \frac{\partial J(\sigma(\mathbf{L}\mathbf{Z}_{\ell-1}\mathbf{W}))}{\partial \mathbf{W}}.
   \end{equation}
  By denoting the feature matrix before and after active as $\mathbf{L}\mathbf{Z}_{\ell-1}\mathbf{W}=\mathbf{Z}^*$, and $\mathbb{S}(\mathbf{Z}^*)=\mathbf{O}$ respectively, Equation~\eqref{eq_chain} can be generalized as
   \begin{equation}\label{eq_chainG1}\small
   \mathbf{W}_{p,q}^{(-)}
   =\eta  \frac{\partial \mathbb{C}(\mathbb{S}(\mathbf{Z}^*))}{\partial \mathbf{W}_{p,q}}.
   \end{equation}
   According to the chain rule and the backpropagation for deep networks with structured layers~[6], Eq~\eqref{eq_chainG1} can be derived as
   \begin{equation}\label{eq_BP_T}\small
      \mathbf{W}_{p,q}^{(-)}
      = \eta  \frac{\partial\mathbb{C}(\mathbb{S}(\mathbf{Z}^*))}{\partial \mathbf{Z}^*}  \frac{\partial \mathbf{Z}^*}  {\partial \mathbf{W}_{p,q}}
      = \eta (\mathbf{L}\mathbf{Z}_{\ell-1})^{\top} \frac{\partial\mathbb{C}(\mathbb{S}(\mathbf{Z}^*))}{\partial \mathbf{Z}^*}.
   \end{equation}
    Since $\mathbb{C}(\mathbb{S}(\mathbf{Z}^*))$ is a scalar, the matrix $\frac{\partial\mathbb{C}(\mathbb{S}(\mathbf{Z}^*))}{\partial \mathbf{Z}^*}$ indexed with $i,j$ is
   \begin{equation}\label{eq_chainG1stItem}\small
     \frac{\partial \mathbb{C}(\mathbb{S}(\mathbf{Z}^*))}{\partial \mathbf{Z}^*_{i,j}}
     =  \frac{\partial \mathbb{C}(\mathbb{S}(\mathbf{Z}^*))}{\partial \mathbb{S}(\mathbf{Z}^*)}
     \frac{\partial \mathbb{S}(\mathbf{Z}^*)}{\partial \mathbf{Z}^*_{i,j}}
     = \sum_{k} \mathbf{Y}_{i,k} \frac{\partial \log \mathbf{O}_{i,k}}{\partial \mathbf{Z}^*_{i,j}}.
   \end{equation}
   Then, by expressing $\mathbf{O}_{i,k}$ as
   \begin{equation}\label{eq_oExpress}\small
     \mathbf{O}_{i,k}=\frac{e^{Z_{i,k}}}{\Omega_{i,\cdot}}, \Omega_{i,\cdot}=\sum_{l}^{e^{Z_{i,l}}} \Longrightarrow \log \mathbf{O}_{i,k} = Z_{i,k} - \log \Omega_{i,\cdot},
   \end{equation}
   we have
   \begin{equation}\label{eq_partialItem1}\small
     \frac{\partial \log \mathbf{O}_{i,k}}{\partial \mathbf{Z}^*_{i,j}}=\delta_{kj}-\frac{1}{\Omega} \frac{\partial \Omega_{i,\cdot}}{\partial \mathbf{Z}^*_{j}},
   \end{equation}
   where $\delta_{kj}$ is the Kronecker delta. Then the softmax-denominator is
   \begin{equation}\label{eq_softmaxDemon}\small
   \frac{\partial \Omega_{i}}{\partial \mathbf{Z}^*_{i,j}}=\sum_{l} e^{\mathbf{Z}^*_{i,l}} \delta_{l j}=e^{\mathbf{Z}^*_{i,j}},
   \end{equation}
    which gives
    \begin{equation}\label{eq_partLpartZ}\small
    \frac{\partial \mathbf{O}_{i,k}}{\partial \mathbf{Z}^*_{i,j}}=\mathbf{O}_{i,k}\left(\delta_{j k}-\mathbf{O}_{i,j}\right).
    \end{equation}
    Note that the derivative is with respect to $Z_{i,j}$, an arbitrary component of $Z$, so the gradient of $\mathbb{C}(\mathbb{S}(Z))$ with respect to $Z$ is then
    \begin{equation}\label{eq_gradJandZ}\small
    \frac{\partial \mathbb{C}(\mathbb{S}(\mathbf{Z}^*))}{\partial \mathbf{Z}^*_{i,j}}
    =\sum_{k} \mathbf{Y}_{i,k}(\mathbf{O}_{i,j}-\delta_{j k})
    =\mathbf{O}_{i,j}(\sum_{k} \mathbf{Y}_{i,k})-\mathbf{Y}_{i,j}.
    \end{equation}
    Since $Y$ is a ont-hot label matrix, $\sum_{k} Y_{i,k}=1$. Therefore, Equation~\eqref{eq_BP_T} can thus be derived as
    \begin{equation}\label{eq_matrixExpres}\small
      \mathbf{W}^{(-)} = \eta (\mathbf{L}\mathbf{Z}_{\ell-1})^{\top}(\mathbf{O}-\mathbf{Y}).
    \end{equation}
    Then, since $\mathbf{L}$ and ${\mathbf{Z}_{\ell-1}} \mathbf{Z}_{\ell-1}^{\top}$ are both symmetric matrix,
    \begin{align}\label{eq_LfLfT}\small
    \mathbf{L}\mathbf{Z}_{\ell-1}\mathbf{W}^{(-)} &= \eta \mathbf{L}\mathbf{Z}_{\ell-1} (\mathbf{L}\mathbf{Z}_{\ell-1})^{\top} (\mathbf{O}-\mathbf{Y}) \notag  \\
    &= \eta (\mathbf{L} {\mathbf{Z}_{\ell-1}} \mathbf{Z}_{\ell-1}^{\top} \mathbf{L})(\mathbf{O}-\mathbf{Y}) \notag \\
    &= \eta ( \mathbf{L} (\mathbf{L} {\mathbf{Z}_{\ell-1}} \mathbf{Z}_{\ell-1}^{\top})^{\top} )^{\top} (\mathbf{O}-\mathbf{Y}).
    \end{align}
    By expanding the dimension of $\mathbf{Z}_{\ell-1}$ by ${\mathbf{Z}_{\ell-1}} \mathbf{Z}_{\ell-1}^{\top}=\mathbf{Z}' \in \mathbb{R}^{N \times N}$ and setting $\mathbf{Z}'$ as the feature matrix which is within a middleware graph $\mathcal{G}'=\{\mathbf{Z}',\mathcal{E}\}$,
    $\mathbf{L} (\mathbf{L} {\mathbf{Z}_{\ell-1}}$, $\mathbf{Z}_{\ell-1}^{\top})^{\top} )^{\top}$ can be modeled as a message passing process of $\mathcal{G}'$ which is driven by MMPM. We use Figure~\ref{fig_mmpm} to illustrate how MMPM works.

    Therefore, Equation~\eqref{eq_DefOfDomain} can be derived as
    \begin{multline}\label{eq_DomainMMPM}\small
    (\mathbf{z}_{i,\ell}^{(t+1)} - \mathbf{z}_{i,\ell}^{(t)}) \cdot (\mathbf{z}_{i,\ell}^{(t+2)} - \mathbf{z}_{i,\ell}^{(t-1)}) \\
    \geq \eta^2 \max_k \| ((\mathcal{P}(\mathbf{Z}_{\ell-1};\mathcal{E})) (\mathbf{O}-\mathbf{Y}))_{k,\cdot} \|^2_2.
    \end{multline}
    Further, proposition 2 proves that node-wise trajectories could be the summation of the edge-wise trajectories of all immediate edges, therefore, through observing the global matrix $O = \mathbb{S}(\mathbf{L}\mathbf{Z}_{\ell-1} \mathbf{W}_\ell^{(t)})$ in epoch $t$ and layer $\ell$, Proposition~\ref{pro_upperbound} can be derived according to Equation~\eqref{eq_DomainMMPM}.

\section{Proof of Proposition~\ref{pro_NTnET}}\label{appen_3}

We first define the \emph{complementary spanning sub graph} (CS-subgraph) $\mathcal{G}^{\text{sub}}_1$ and $\mathcal{G}^{\text{sub}}_2$ of $\mathcal{G}=\{ \mathbf{Z}, \mathcal{E} \}$ as
\begin{myDef}[Complementary Spanning Subgraph]
If $\mathcal{G}^{\text{sub}}_1$ and $\mathcal{G}^{\text{sub}}_2$ are the spanning sub-graph of $\mathcal{G}$, meanwhile they satisfy
 \begin{multline}\label{eq_subG}\small
     \mathcal{G}^{\text{sub}}_1 = \{ \mathbf{Z},\mathcal{E}^{\text{sub}}_1 \}, \mathcal{G}^{\text{sub}}_2 = \{ \mathbf{Z},\mathcal{E}^{\text{sub}}_2 \} \\
     \text{s.t. } \mathcal{E}^{\text{sub}}_1 \cup \mathcal{E}^{\text{sub}}_2 = \mathcal{E}, \mathcal{E}^{\text{sub}}_1 \cap \mathcal{E}^{\text{sub}}_2 = \phi, \\
     \mathbf{L}_{\mathcal{G}^{\text{sub}}_1} + \mathbf{L}_{\mathcal{G}^{\text{sub}}_2} = \mathbf{L}_{\mathcal{G}},
\end{multline}
they are the complementary spanning sub-graph of $\mathcal{G}$.
\label{def_SubGraph}
\end{myDef}
Figure~\ref{fig_SubGraph} shows an example of the CS-subgraph.
\begin{figure}[htb]
\centering
\includegraphics[width=0.45\textwidth]{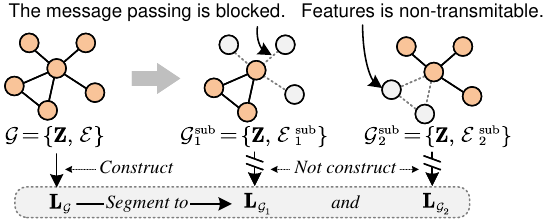}
\caption{An illustration of CS subgraph.}
\label{fig_SubGraph}
\end{figure}
According to the Definition~\ref{def_SubGraph}, the global message passing on $\mathcal{G}$ can be the additive sum of the CS-subgraph of $\mathcal{G}$, i.e.,
\begin{equation}\label{eq_SumofSubgraph}\small
\mathbf{L}_{\mathcal{G}}\mathbf{Z} = \mathbf{L}_{\mathcal{G}^{\text{sub}}_1}\mathbf{Z} + \mathbf{L}_{\mathcal{G}^{\text{sub}}_2}\mathbf{Z},
\end{equation}
we can thus constantly segment $\mathcal{G}$ and active message passing on its CS-subgraphs, i.e.,
\begin{align}\label{eq_SubSubSubGraphs}\small
\mathbf{L}_{\mathcal{G}}\mathbf{Z}
&= \mathbf{L}_{ \{ \mathbf{Z},\mathcal{E}^{\text{sub}}_1 \} }\mathbf{Z} +  \mathbf{L}_{ \{ \mathbf{Z},\mathcal{E}^{\text{sub}}_2 \} }\mathbf{Z}
= \ldots
= \sum_{\ell \in \mathcal{E}} \mathbf{L}_{\{ \mathbf{Z},\mathcal{E} \}} \mathbf{Z}. \notag \\
&=\sum_{\ell \in \mathcal{E}} \mathcal{R}(\mathbf{L},l))\mathbf{Z},
\end{align}
since the definition of $\mathcal{R}(\cdot , \cdot)$ supposes that $\sum_{l \in \mathcal{E}}=\mathbf{L}$. Equation~\eqref{eq_NTnET} can thus be obtained by substituting Equation~\eqref{eq_SubSubSubGraphs} into Equation~\eqref{eq_DefEdgeT}.

\section*{References}
\begin{itemize}
  \item [\lbrack1\rbrack] Petersen Philipp, Raslan Mones, Voigtlaender Felix. Topological properties of the set of functions generated by neural networks of fixed size. \emph{Foundations of Computational Mathematics}. 2021, 21: 375-444.
  \item [\lbrack2\rbrack] Rosenblatt Murray. A central limit theorem and a strong mixing condition. \emph{PNAS}. 1956, 42(1):43-47.
  \item [\lbrack3\rbrack] Loukas Andreas. What graph neural networks cannot learn: depth vs width. \emph{ICLR}, 2019, May.
  \item [\lbrack4\rbrack] Kullback Solomon, Leibler Richard A. On information and sufficiency. \emph{The Annals of Mathematical Statistics}. 1951, 22(1):79-86.
  \item [\lbrack5\rbrack] Cortes Corinna, Vapnik Vladimir. Support-vector networks. \emph{Machine learning}. 1995, 20:273-297.
\end{itemize}

\end{document}